\documentclass[number]{ReportTemplate}
\usepackage[colorlinks,urlcolor=blue,linkcolor=blue,citecolor=blue]{hyperref}
\usepackage{utopia}
\usepackage{amssymb}
\usepackage{amsfonts}
\usepackage{amsmath}
\usepackage{mathtools}
\usepackage{amsthm}
\usepackage{bm}
\usepackage{graphicx}
\usepackage{algorithm}
\usepackage{algorithmic}
\usepackage[american]{babel}
\usepackage{subfigure}
\usepackage{hyperref}
\mathtoolsset{showonlyrefs,showmanualtags}
\usepackage{color}

\usepackage[margin=1.0in]{geometry}

\usepackage{enumerate}

\usepackage{pgfplots}
\usepackage{epsfig}

\usepackage{tcolorbox}
\usepackage{mdframed}
\usepackage{lipsum}

\renewcommand{\hat}{\widehat}

\usepackage{multirow}

\newtheorem{remark}{Remark}

\renewenvironment{myproof}[1]
{\par\noindent\textbf{Proof of #1.}\ \enspace\ignorespaces\begin{allowdisplaybreaks}}
{\end{allowdisplaybreaks}\hspace{\stretch{1}}$\square$}

\pgfplotsset{compat=1.18} 

\begin{document}
\begin{frontmatter}
\title{Migrant Resettlement\\ by Evolutionary Multi-objective Optimization}

\author{Dan-Xuan Liu\textsuperscript{\rm 1,2}}
\ead{liudx@lamda.nju.edu.cn}

\author{Yu-Ran Gu\textsuperscript{\rm 1,2}}
\ead{guyr@lamda.nju.edu.cn}

\author{Chao Qian\textsuperscript{\rm 1,2}*}
\ead{qianc@lamda.nju.edu.cn}
\cortext[cor1]{This work was supported by the National Science Foundation of China (62276124). Chao Qian is the corresponding author.}

\author{Xin Mu\textsuperscript{\rm 3}}
\ead{mux@pcl.ac.cn}

\author{Ke Tang\textsuperscript{\rm 4}}
\ead{tangk3@sustech.edu.cn}

\address{\textsuperscript{\rm 1}National Key Laboratory for Novel Software Technology, Nanjing University, Nanjing 210023, China\\ \vspace{0.5em}
\textsuperscript{\rm 2}School of Artificial Intelligence, Nanjing University, Nanjing 210023, China\\\vspace{0.5em}
\textsuperscript{\rm 3}Peng Cheng Laboratory, Shenzhen 518000, China\\\vspace{0.5em}
\textsuperscript{\rm 4}Shenzhen Key Laboratory of Computational Intelligence, Department of Computer Science and Engineering, \\Southern University of Science and Technology, Shenzhen 518055, China
}
\begin{abstract}
Migration has been a universal phenomenon, which brings opportunities as well as challenges for global development. As the number of migrants (e.g., refugees) increases rapidly in recent years, a key challenge faced by each country is the problem of migrant resettlement. This problem has attracted scientific research attention, from the perspective of maximizing the employment rate. Previous works mainly formulated migrant resettlement as an approximately submodular optimization problem subject to multiple matroid constraints and employed the greedy algorithm, whose performance, however, may be limited due to its greedy nature. In this paper, we propose a new framework MR-EMO based on Evolutionary Multi-objective Optimization, which reformulates Migrant Resettlement as a bi-objective optimization problem that maximizes the expected number of employed migrants and minimizes the number of dispatched migrants simultaneously, and employs a Multi-Objective Evolutionary Algorithm (MOEA) to solve the bi-objective problem. We implement MR-EMO using three MOEAs, the popular NSGA-II, MOEA/D as well as the theoretically grounded GSEMO. To further improve the performance of MR-EMO, we propose a specific MOEA, called GSEMO-SR, using matrix-swap mutation and repair mechanism, which has a better ability to search for feasible solutions. We prove that MR-EMO using either GSEMO or GSEMO-SR can achieve better theoretical guarantees than the previous greedy algorithm. Experimental results under the interview and coordination migration models clearly show the superiority of MR-EMO (with either NSGA-II, MOEA/D, GSEMO or GSEMO-SR) over previous algorithms, and that using GSEMO-SR leads to the best performance of MR-EMO.
\end{abstract}


\end{frontmatter}

\newpage
\section{Introduction}

As an important part of globalization, the phenomenon of migration has become an unstoppable trend. In 2019, the number of international migrants was 272 million, exceeding predictions for the year 2050~\cite{migration-report20}. Migration is beneficial to the sustainable development of both host and home countries, as well as migrants themselves. For host countries, massive migrants can fill the manpower shortage~\cite{host-country1}, and entrepreneurs among them can even create new employment opportunities~\cite{host-country3}. What behind the higher employment rates is the prosperity and development of all walks of life~\cite{host-country2}. Meanwhile, migrants can earn more salaries, and bring back more foreign exchange and trade for home countries~\cite{home-country1}.

Migration also brings challenges, one of which is migrant resettlement. For example, there were 82.4 million refugees in the world in 2020, more than double the level in 2010~\cite{UNHCR2019}. Because refugees may lack language skills and have difficulties in adaptation, it is full of challenges and obstacles for policymakers to make them find jobs in the host countries~\cite{marbach2018long}, without developed systematic and pellucid resettlement procedures~\cite{delacretaz2016refugee}, which consequently have attracted scientific research attention.

A paper published in \textit{Science} first tried to select a subset of migrant-locality pairs maximizing the expected number of employed migrants, through a data-driven algorithm combining supervised machine learning with optimal matching~\cite{bansak2018improving}. Utilizing historical registry data, supervised learning models are built to predict the employment probability of each migrant at each locality. By assuming that the employment of different migrants at the same locality is independent, the algorithm can lead to a $40\%$-$70\%$ migrants' employment gain relative to the actual outcomes.

The above additive algorithm~\cite{bansak2018improving} ignores the effect of competition in the real-world labor market. That is, whether one migrant gets a job or not is also influenced by other migrants' employment. Thus, G{\"o}lz and Procaccia~\cite{golz2019migration} introduced two migration models, i.e., interview and coordination models, to simulate the competition among migrants. For these models, the objective function, i.e., the expected number of migrants who find employment, all satisfies the $\epsilon$-approximately submodular property, where $\epsilon \geq 0$. Furthermore, the constraints that each migrant can only go to one locality to apply for a job and each locality can accommodate a limited number of migrants, can be characterized by the intersection of $k$ matroids, where $k\geq 2$. Thus, migrant resettlement is cast as the general problem of $\epsilon$-approximately submodular maximization subject to $k$ matroid constraints. G{\"o}lz and Procaccia~\cite{golz2019migration} proposed the greedy algorithm, which achieves an approximation guarantee of $1/(k+1+\frac{4\epsilon r}{1-\epsilon})$, where $r$ is the size of the largest feasible subset of migrant-locality pairs. They empirically showed that the greedy algorithm can bring improvement on both models, compared with the additive algorithm~\cite{bansak2018improving}.

\begin{figure*}[t!]\centering
\begin{minipage}[c]{0.36\linewidth}\centering
        \includegraphics[width=0.8\linewidth]{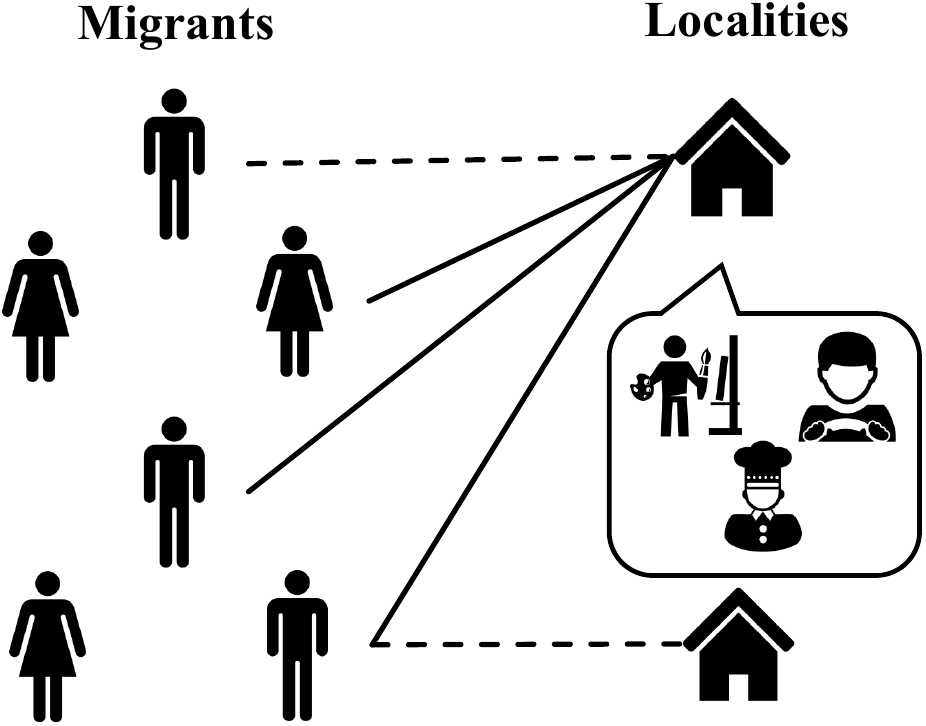}
\end{minipage}
\begin{minipage}[c]{0.44\linewidth}\centering
        \includegraphics[width=0.9\linewidth]{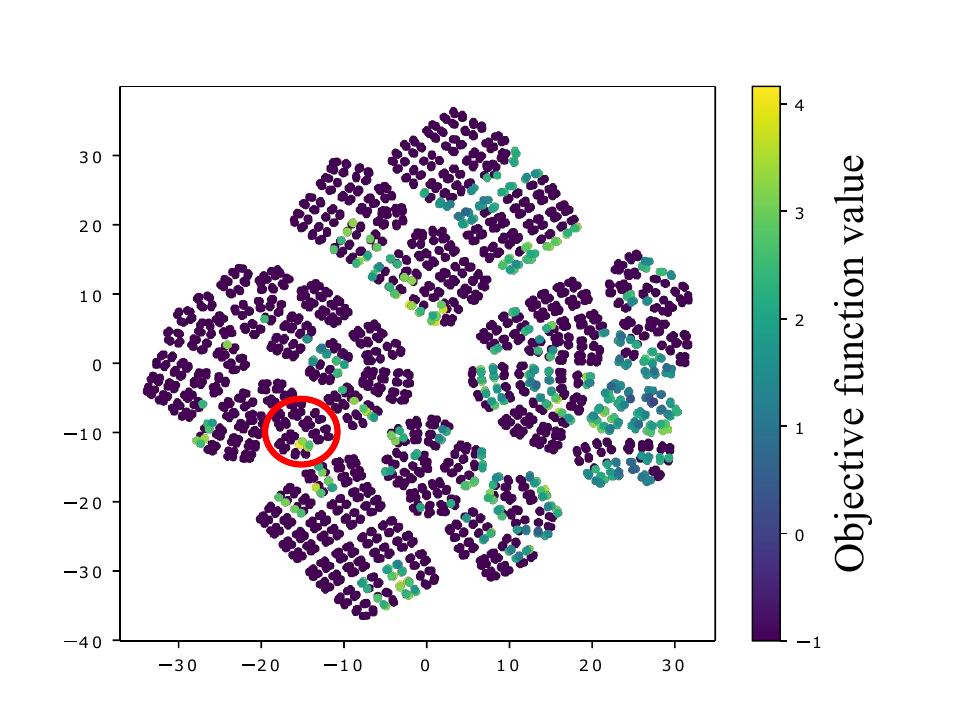}
\end{minipage}\\
\begin{minipage}[c]{0.36\linewidth}\centering
    \small(a) 
\end{minipage}
\begin{minipage}[c]{0.44\linewidth}\centering
    \small(b) 
\end{minipage}
\caption{(a) An illustration of a migrant resettlement problem instance, where each edge (migrant-locality pair) denotes that the migrant is assigned to the locality. (b) The heat map of a migrant resettlement problem instance (where six migrants go to two localities to find jobs under the interview migration model) in the solution space. The color of a point corresponds to the objective function value, i.e., the expected number of migrants who find employment, of a solution; a point with the darkest color (i.e., the value $-1$) implies an infeasible solution.}\label{heatmap}
\end{figure*}

Considering that the performance of the greedy algorithm may be limited due to its greedy nature, in this paper, we propose a new framework based on Evolutionary Multi-objective Optimization~\cite{knowles2001reducing,ecj15submodular,qian2015subset} for Migrant Resettlement, briefly called MR-EMO, which reformulates migrant resettlement as a bi-objective optimization problem that maximizes the $\epsilon$-approximately submodular objective value of a feasible subset of migrant-locality pairs and minimizes the subset size simultaneously. That is, MR-EMO tries to maximize the expected number of employed migrants while dispatching as few migrants as possible. MR-EMO can be equipped with any MOEA to solve this bi-objective problem, and we employ the popular NSGA-II~\cite{nsgaii}, MOEA/D~\cite{moead} as well as the theoretically grounded GSEMO~\cite{Laumanns04}. Empirical results on the interview and coordination migration models show that using either NSGA-II, MOEA/D or GSEMO, MR-EMO performs better than previous algorithms.


Next, we try to design a specific MOEA for the migrant resettlement problem, bringing further performance improvement. Figure~\ref{heatmap}(a) gives an illustration of a migrant resettlement problem instance, in which six migrants go to two localities to find jobs. Let $V=\{v_1,\ldots,v_{|V|}\}$ and $L=\{l_1,\ldots,l_{|L|}\}$ denote a set of migrants and localities, respectively. A migrant-locality pair $(v,l)$ means that the migrant $v\in V$ is assigned to the locality $l\in L$. Each (solid or dotted) line in Figure~\ref{heatmap}(a) denotes a migrant-locality pair. Assume that the capacity of each locality is three. The two matroid constraints are that each migrant can only go to one locality and each locality can accommodate at most three migrants. When we consider the three solid lines in Figure~\ref{heatmap}(a) only, it is a feasible solution, i.e., satisfies the constraints. But if adding the two dotted lines, the constraints will be violated, since there is a migrant assigned to two localities and also a locality absorbing four migrants more than its capacity three. The solution violating these two matroid constraints is infeasible. Figure~\ref{heatmap}(b) gives the heat map of this migrant resettlement instance in the solution space, that is, the heat map of six migrants going to two localities to find jobs under the interview migration model. There are $6\times2$ migrant-locality pairs in total, and the solution of this instance can be naturally represented by a Boolean vector $\bm x\in\{0,1\}^{12}$, where the $i$-th bit $x_i=1$ iff the $i$-th migrant-locality pair is selected. We reduce the dimension of solutions from 12d to 2d using the t-SNE method~\cite{van2008visualizing}, which can ensure that the solutions with closer positions in the original space also have closer positions in the 2d projection space. In Figure~\ref{heatmap}(b), the color of a point corresponds to the objective function value, i.e., the expected number of migrants who find employment, of a solution; a point with the darkest color (i.e., the value $-1$) implies an infeasible solution. Then we can have two observations:\\
1. The whole space is filled with a large number of infeasible solutions; \\
2. Good feasible solutions may be surrounded by infeasible ones, as indicated by the circle. 

Observation 1 reflects the character of solutions subject to two matroid constraints. Let $n$ denote the total number of migrant-locality pairs, i.e., $n=|V| \cdot |L|$. For any feasible solution $\bm x\in\{0,1\}^{n}$, if the bit corresponding to migrant-locality pair $(v_i,l_j)$ is 1, then any bit corresponding to migrant-locality $(v_i,l_k)$, where $k\neq j$, must be 0. Otherwise, the first matroid constraint, that each migrant can only go to one locality to apply for a job, will be violated. Note that $\bm x$ also needs to satisfy the second matroid constraint, that is, for any locality $l_j$, the number of migrants assigned to locality $l_j$ cannot exceed the capacity of the corresponding locality, which limits the number of 1s in those bits corresponding to $(\cdot, l_j)$. However, the classical bit-wise mutation operator flips each bit of $\bm x$ independently with equal probability, which will probably generate an infeasible solution rather than a better feasible solution.

To better explore the feasible region, we propose the matrix-swap mutation operator. We first represent a solution $\bm x\in\{0,1\}^n$ by a Boolean matrix of size $|V|\times |L|$, where $|V|$ and $|L|$ denote the number of migrants and localities, respectively, and $n=|V|\times |L|$. The matrix representation of a feasible solution contains at most one 1-bit in each row, and no more 1-bits in each column than the capacity of the corresponding locality. Matrix-swap mutation selects two rows (or two columns) of a parent solution uniformly at random and swaps them to generate an offspring, which can better inherit the feasibility of the parent solution to reduce the probability of generating an infeasible offspring. Observation 2 shows that good feasible solutions may be surrounded by infeasible ones, implying that infeasible solutions may be also valuable. Thus, we introduce a repair mechanism, which randomly flips some migrant-locality bits of an infeasible offspring solution from 1 to 0 to fix it into feasible, instead of discarding it directly.


By incorporating the proposed matrix-Swap mutation operator and the Repair mechanism into GSEMO, we propose a new MOEA, called GSEMO-SR. Empirical results show that MR-EMO using GSEMO-SR is better than previous algorithms. Moreover, MR-EMO using GSEMO-SR is significantly better than using NSGA-II, MOEA/D or GSEMO. The ablation experimental results show that GSEMO-SR employing only matrix-swap mutation or repair mechanism can also improve the performance of MR-EMO, and applying both of them performs the best, which verifies the effectiveness of the operators. Note that the matrix-swap mutation operator and repair mechanism can be combined with other MOEAs, which may also bring performance improvement.

In addition to better empirical results, MR-EMO can also achieve a theoretical advantage. We prove that MR-EMO using no matter GSEMO or GSEMO-SR can achieve an approximation guarantee of $1/(k+\frac{1}{p}+\frac{2\epsilon r}{1-\epsilon})$, which is better than $1/(k+1+\frac{4\epsilon r}{1-\epsilon})$ of the greedy algorithm~\cite{golz2019migration}, where $p \geq 1$. Here, an approximation guarantee of $\alpha$ (where $0<\alpha<1$) is the worst-case theoretical guarantee, which is commonly used in the theoretical analysis of optimization algorithms. It implies that for any instance of migrant resettlement, the solution generated by the algorithm has an objective value of at least $\alpha \cdot \mathrm{OPT}$, where $\mathrm{OPT}$ denotes the optimal function value. 

Though the proposed framework MR-EMO solves the migrant resettlement problem by considering the expected number of employed migrants only, the calculation of the expected number of employed migrants depends on the matching probability of each migrant with each job, which may be generated by considering migrants’ own technology, migrant wishes, local policies that support migrants, etc, during data preparation. For example, supervised learning models can be built to predict the matching probability, by considering migrants’ background characteristics from real-world data~\cite{bansak2018improving}.

The rest of the paper is organized as follows. Section~\ref{prob-migration} first introduces some preliminaries on the studied problem, i.e., migrant resettlement. Section~\ref{framework} then presents the general MR-EMO framework and the specific MOEA (i.e., GSEMO-SR) for migrant resettlement. Sections~\ref{Theo} and~\ref{exp} give the theoretical analysis and empirical study, respectively. Section~\ref{conclusion} finally concludes this paper.


\section{Migrant Resettlement}\label{prob-migration}

Let $V=\{v_1,\ldots,v_{|V|}\}$ and $L=\{l_1,\ldots,l_{|L|}\}$ denote a set of migrants and localities, respectively. A migrant-locality pair $(v,l)$ means that the migrant $v \in V$ is assigned to the locality $l \in L$. For example, there are six migrants and two localities in Figure~\ref{heatmap}(a), and each (solid or dotted) line denotes a migrant-locality pair. Let $N=V\times L$ denote the set of all migrant-locality pairs. The migrant resettlement problem is to select a subset $X$ of migrant-locality pairs (i.e., $X \subseteq N$), that maximizes the expected number of employed migrants.

Suppose that the number of job opportunities at locality $l$ is $\mathrm{job}_l$. As there are different professions, e.g., drivers, chefs and painters in Figure~\ref{heatmap}(a), $\mathrm{job}_l$ is further divided, i.e., $\mathrm{job}_{l}=\sum_{\pi \in \Pi} \mathrm{job}_{l,\pi}$, where $\Pi$ denotes the set of professions and $\mathrm{job}_{l,\pi}$ denotes the number of job opportunities for profession $\pi$ at locality $l$. To simulate the competition among migrants, G{\"o}lz and Procaccia~\cite{golz2019migration} have introduced two migration models, i.e., interview and coordination models, which will be introduced in the following, respectively. We will also show how to calculate the objective function, i.e., the expected number of migrants who find employment, for each model.

\textbf{Interview Model.} Migrants can only find jobs that match their types of professions, that is, competition exists among migrants with the same profession. Given a set $X$ of selected migrant-locality pairs, let $V_{l,\pi}=\{v\mid (v,l)\in X, \text{and the profession of $v$ is $\pi$}\}$, i.e., the set of migrants with profession $\pi$ assigned to locality $l$. To allocate the jobs of profession $\pi$ at locality $l$, the interview model assumes that there is a specific order for the migrants in $V_{l,\pi}$, and each one will be interviewed sequentially. For $1\leq i\leq |V_{l,\pi}|$, let $v^i_{l,\pi}$ denote the $i$-th migrant in $V_{l,\pi}$. When interviewing the migrant $v^i_{l,\pi}$, we use $\mathrm{job}_{l,\pi}^{i}$ to denote the remaining number of available jobs of profession $\pi$ at locality $l$. The matching probability of migrant $v^i_{l,\pi}$ and each of the $\mathrm{job}_{l,\pi}^{i}$ available jobs is $p_{v^i_{l,\pi},l}$, and $v^i_{l,\pi}$ will try to apply for each of the $\mathrm{job}_{l,\pi}^{i}$ available jobs. If there is at least one success in these $\mathrm{job}_{l,\pi}^{i}$ attempts, $v^i_{l,\pi}$ will be employed. Thus, the employment of $v^i_{l,\pi}$ can be represented by an indicator function $\mathbb{I}(\mathrm{job}_{l,\pi}^{i},p_{v^i_{l,\pi},l})$, which equals 1 iff the binomially distributed random variable $\mathcal{B}(\mathrm{job}_{l,\pi}^{i},p_{v^i_{l,\pi},l})$ is at least 1. At the beginning, $\mathrm{job}_{l,\pi}^{1}=\mathrm{job}_{l,\pi}$, i.e., the initial number of job opportunities for profession $\pi$ at locality $l$. As the interview process goes, $\forall i >1: \mathrm{job}_{l,\pi}^{i}=\mathrm{job}_{l,\pi}^{i-1}-\mathbb{I}(\mathrm{job}_{l,\pi}^{i-1},p_{v^{i-1}_{l,\pi},l})$. Then, the number of migrants with profession $\pi$ that find jobs at locality $l$ can be represented as $\sum^{|V_{l,\pi}|}_{i=1}\mathbb{I}(\mathrm{job}_{l,\pi}^{i},p_{v^i_{l,\pi},l})$, and the expected total number of employed migrants for the interview model is
\begin{align}\label{eq-interview-model}f(X)=\sum_{l\in L}\sum_{\pi\in\Pi} \mathbb{E} \left[\sum_{i=1}^{|V_{l,\pi}|}\mathbb{I}(\mathrm{job}_{l,\pi}^{i},p_{v^i_{l,\pi},l})\right],\end{align}
where the order of migrants in $V_{l,\pi}$ is determined randomly.

\textbf{Coordination Model.} Different from the above interview model where a migrant can only find a job of her profession, the coordination model can coordinate the assignment by compatibility between all migrants and jobs at locality $l$, demonstrating a more realistic situation. The matching probability of each migrant $v$ and a job of profession $\pi$ is $p_{v,\pi}$, which breaks the strict restriction between professions, because migrants also have opportunities to work in other professional jobs related to their careers or less skilled. Given a set $X$ of selected migrant-locality pairs, let $V_l =\{v\mid (v,l)\in X\}$ denote the set of migrants assigned to locality $l$, i.e., $V_l=\cup_{\pi\in \Pi} V_{l,\pi}$. We construct a bipartite graph $B(V_l)$, where the two sets of vertices are $V_l$ and the $\mathrm{job}_l$ jobs, respectively. For each $v\in V_l$ and each job $j$, we add one edge between them with probability $p_{v,\pi_j}$, where $\pi_j$ denotes the profession of job $j$. Then, the number of migrants that find jobs at locality $l$ is calculated as the size of maximum matching of the bipartite graph $B(V_l)$. Thus, the expected number of employed migrants for the coordination model is
\begin{align}\label{eq-coord-model}
f(X)=\sum_{l\in L} \mathbb{E} \left[ \text{size of maximum matching of $B(V_l)$}\right].
\end{align}

The objective functions, defined as Eqs. (1) and (2), of the interview model and coordination model are all monotone submodular~\cite{golz2019migration}. A set function $f: 2^N \rightarrow \mathbb{R}$ is monotone if $\forall X \subseteq Y \subseteq N$, $f(X) \leq f(Y)$. As more migrant-locality pairs will not worsen the value, the monotonicity is satisfied naturally. Assume w.l.o.g. that monotone functions are normalized, i.e., $f(\emptyset)=0$. A set function $f$ is submodular~\cite{nemhauser1978analysis} if it satisfies the ``diminishing returns'' property, i.e., $\forall X\subseteq Y\subseteq N=V\times L, (v,l) \notin Y$, \begin{align}\label{eq-submodular-1}f(X\cup \{(v,l)\})-f(X)\ge f(Y\cup\{(v,l)\})-f(Y),\end{align} or equivalently $\forall X\subseteq Y\subseteq N$, \begin{align}\label{eq-submodular-2}f(Y)-f(X)\leq \sum_{(v,l)\in Y\setminus X}(f(X\cup\{(v,l)\})-f(X)).\end{align}
The diminishing returns property Eq.~(\refeq{eq-submodular-1}) reflects the competition effects naturally, i.e., the chance of migrant $v$ to find a job at locality $l$ decreases as the number of migrants competing for jobs increases.

For the migrant resettlement problem, there are also two constraints to be satisfied. One is that each migrant can only be distributed to at most one locality and get one job, and the other is that each locality $l$ can accept at most $\mathrm{cap}_l$ migrants. $\mathrm{cap}_l$ is called the capacity of locality $l$, and can be different from $\mathrm{job}_l$ (i.e., the number of job opportunities at locality $l$), because some jobs may be allocated to natives. Let us see Figure~\ref{heatmap}(a) for an illustration. Assume that the capacity of each locality is three. When we consider the three solid lines only, it is a feasible solution, i.e., satisfies the constraints. But if adding the two dotted lines, the constraints will be violated, since there is a migrant assigned to two localities and also a locality absorbing four migrants more than its capacity three.

Let $\{N^{\mathrm{mig}}_1,\ldots,N^{\mathrm{mig}}_{|V|}\}$ and $\{N^{\mathrm{loc}}_1,\ldots,N^{\mathrm{loc}}_{|L|}\}$ denote two partitions of $N$, where $N^{\mathrm{mig}}_i=\{(v_i,l) \mid l\in L\}$ and $N^{\mathrm{loc}}_i=\{(v,l_i)  \mid v\in V\}$ are the set of migrant-locality pairs with migrant $v_i$ and locality $l_i$, respectively. Then, the two constraints can be represented as $\forall 1\leq i\leq |V|: |X\cap N^{\mathrm{mig}}_i|\leq 1$ and $\forall 1\leq i\leq |L|: |X\cap N^{\mathrm{loc}}_i| \leq \mathrm{cap}_i$, which are actually two partition matroids. A matroid is a pair $(N,\mathcal{F}\subseteq 2^N)$, satisfying the hereditary (i.e., $\emptyset\in\mathcal{F}$ and $\forall X\subseteq Y\in \mathcal{F}:X\in\mathcal{F}$) and augmentation (i.e., $\forall X,Y\in \mathcal{F}$, $|X|>|Y|: \exists (v,l) \in X \setminus Y, Y\cup\{(v,l)\}\in \mathcal{F}$) properties~\cite{calinescu2011maximizing}. Thus, the constraints of migrant resettlement can be viewed as the intersection of two specific matroids. As observed from Figure~\ref{heatmap}(b), these two matroid constraints will lead to a large number of infeasible solutions.

In~\cite{golz2019migration}, migrant resettlement has been generally formulated as the problem of maximizing a monotone submodular function subject to multiple matroid constraints, as shown in Definition~\ref{def-problem}, where the monotone submodular objective function characterizes the expected number of employed migrants under different models.

\begin{definition}[Migrant Resettlement]\label{def-problem}
Given all migrants $V=\{v_1,\ldots,v_{|V|}\}$ and localities $L=\{l_1,\ldots,l_{|L|}\}$, a monotone submodular function $f:2^{N} \rightarrow \mathbb{R}^+$, and $k$ matroids $\mathcal{F}_1,\ldots,\mathcal{F}_k\subseteq 2^N$, the goal of migrant resettlement is to find a subset $X\subseteq N$ of migrant-locality pairs such that
\begin{equation}
\begin{aligned}\label{eq:subsetsel}
\mathop{\arg\max}\nolimits_{X\subseteq N} f(X) \quad \text{s.t.}\quad X\in \bigcap\nolimits_{i=1}^{k}\mathcal{F}_{i},
\end{aligned}
\end{equation}
where $N=V\times L$ is the set of all migrant-locality pairs.
\end{definition}

However, we often cannot obtain the exact value of the objective $f$ in practice, but only a noisy one, denoted as $\hat{f}$. For example, the objective functions in Eqs.~(\refeq{eq-interview-model}) and~(\refeq{eq-coord-model}) are all the expectation of a complicated random variable, which can be estimated only by the average of multiple sampling, resulting in the inexactness and even the violation of monotonicity and submodularity. Thus, the objective obtained by algorithms is assumed to be $\hat{f}$ instead of $f$, which only satisfies the $\epsilon$-approximate submodularity in Definition~\ref{def-epsilon-app-submodular}. That is, the value of $\hat{f}$ is bounded between $(1-\epsilon)\cdot f$ and $(1+\epsilon)\cdot f$. When $\epsilon=0$, $\hat{f}$ is exactly the true objective $f$.

\begin{definition}[$\epsilon$-Approximate Submodularity~\cite{horel2016maximization,qian2019maximizing}]\label{def-epsilon-app-submodular}Let $\epsilon\ge 0$. A set function $\hat{f}:2^{N}\rightarrow \mathbb{R}$ is $\epsilon$-approximately submodular if there exists a monotone submodular function $f$ such that $\forall X\subseteq N$,
\begin{align}\label{eq:epsilon}
    (1-\epsilon)\cdot f(X)\leq \hat{f}(X) \leq (1+\epsilon)\cdot f(X).
\end{align}
\end{definition}

For the migrant resettlement problem in Definition~\ref{def-problem}, G{\"o}lz and Procaccia~\cite{golz2019migration} proposed the greedy algorithm, which starts from the empty set, and iteratively adds one migrant-locality pair with the largest marginal gain on $\hat{f}$ while satisfying the constraints. The greedy algorithm achieves a $(1/(k+1+\frac{4\epsilon r}{1-\epsilon}))$-approximation guarantee, where $r$ denotes the size of the largest feasible solution, i.e., the largest subset of $N$ belonging to $\bigcap_{i=1}^{k}\mathcal{F}_{i}$. That is, for any instance of migrant resettlement, the objective function value of the solution $X$ generated by the greedy algorithm satisfies
\begin{align*}
f(X) \geq \left(1\big/\left(k+1+\frac{4\epsilon r}{1-\epsilon}\right)\right) \cdot \mathrm{OPT}, 
\end{align*}
where $\mathrm{OPT}$ denotes the optimal function value.

\section{MR-EMO Framework}\label{framework}
Inspired by the excellent performance of MOEAs for solving general subset selection problems~\cite{ecj15submodular,qian2015subset,QianS0TZ17,qian2019maximizing,Qian20,NeumannN20,roostapour2022pareto}, we propose a new Migrant Resettlement framework based on Evolutionary Multi-objective Optimization, called MR-EMO. Assume that $|N|=|V|\cdot |L|=n$, i.e., there are totally $n$ pairs of migrant-locality. A subset $X$ of $N$ can be naturally represented by a Boolean vector $\bm{x} \in\{0,1\}^n$, where the $i$-th bit $x_{i}=1$ iff the $i$-th migrant-locality pair in $N$ is contained by $X$. In the following, we will not distinguish $\bm{x}\in\{0,1\}^n$ and its corresponding subset of migrant-locality pairs for notational convenience.

As presented in Algorithm~\ref{alg:MR-EMO}, MR-EMO first reformulates the original migrant resettlement problem in Definition~\ref{def-problem} as a bi-objective maximization problem 
\begin{align}\label{def-CO-BO}
\arg\max\nolimits_{\bm{x} \in \{0,1\}^n}& \quad  (f_1(\bm{x}),f_2(\bm{x})),
\end{align}
where
\begin{align*}
f_1(\bm{x}) = \begin{cases}
	\hat{f}(\bm{x}), &{\bm{x} \in \bigcap_{i=1}^{k}\mathcal{F}_{i}}\\
	-1, &{\text{otherwise}}
\end{cases},\quad
f_{2}(\bm{x})=|\bm{x}|_{0}.
\end{align*}
That is, the first objective $f_1$ equals the original objective $\hat{f}$ (i.e., the expected number of employed migrants) for feasible solutions (i.e., solutions satisfying the constraints), while $-1$ for infeasible ones; the second objective $f_2$ equals the number of 0-bits of a solution. As the two objectives may be conflicting, the domination relationship in Definition~\ref{def:domination} is often used for comparing two solutions. A solution is Pareto optimal if no other solution dominates it. The collection of objective vectors of all Pareto optimal 
solutions is called the Pareto front.

\begin{definition}[Domination]\label{def:domination}
For two solutions $\bm{x}$ and $\bm{x}'$,\\
1. $\bm{x}$ weakly dominates $\bm{x}'$ (i.e., $\bm{x}$ is better than $\bm{x}'$, denoted by $\bm{x}\succeq \bm{x}'$) if $\ \forall i: f_{i}(\bm{x})\ge f_{i}(\bm{x}')$;\\
2. $\bm{x}$ dominates $\bm{x}'$ (i.e., $\bm{x}$ is strictly better than $\bm{x}'$, denoted by $\bm{x}\succ \bm{x}'$) if $\bm{x}\succeq \bm{x}' \land \exists i: f_{i}(\bm{x})> f_{i}(\bm{x}')$;\\
3. $\bm{x}$ and $\bm{x}'$ are incomparable if neither $\bm{x}\succeq \bm{x}'$ nor $\bm{x}'\succeq \bm{x}$.
\end{definition}

After constructing the bi-objective problem in Eq.~(\refeq{def-CO-BO}), MR-EMO employs an MOEA to solve it, as shown in line~2 of Algorithm~\ref{alg:MR-EMO}. Evolutionary algorithms (EAs), inspired by Darwin’s theory of evolution, are general-purpose randomized heuristic optimization algorithms~\cite{back:96}, mimicking variational reproduction and natural selection. Starting from an initial population of solutions, EAs iteratively reproduce offspring solutions by crossover and mutation, and select better ones from the parent and offspring solutions to form the next population. The population-based search of EAs matches the requirement of multi-objective optimization, i.e., EAs can generate a set of Pareto optimal solutions by running only once. Thus, EAs have become the most popular tool for multi-objective optimization~\cite{coello2007evolutionary,he2021survey,hong2021evolutionary,liang2024evolutionary,yu2022survey}, and the corresponding algorithms are also called MOEAs. After running a number of iterations, the best solution, i.e., $\arg\max_{\bm{x} \in P, \bm{x}\in\bigcap_{i=1}^{k}\mathcal{F}_{i}} \hat{f}(\bm{x})$ will be selected from the final population as the output, as shown in line~3 of Algorithm~\ref{alg:MR-EMO}. Note that the aim of MR-EMO is to find a good solution of the original migrant resettlement problem in Definition~\ref{def-problem}, rather than the Pareto front of the reformulated bi-objective problem in Eq.~(\ref{def-CO-BO}). That is, the bi-objective reformulation is an intermediate process. The introduction of the second objective $f_2$ can naturally bring a diverse population, which may lead to better optimization performance. 

\begin{algorithm}[t!]\caption{MR-EMO Framework}\label{alg:MR-EMO}
\textbf{Input}: a migrant resettlement problem instance\\
\textbf{Output}: a subset of migrant-locality pairs satisfying multiple matroid constraints\\
\textbf{Process}:
    \begin{algorithmic}[1]
    \STATE Construct the two objective functions $f_1(\bm{x})$ and $f_2(\bm{x})$ to be maximized, as presented in Eq.~(\refeq{def-CO-BO});
    \STATE Apply an MOEA to solve the bi-objective problem;
    \STATE \textbf{return} the best feasible solution in the final population generated by the MOEA
    \end{algorithmic}
\end{algorithm}

MR-EMO can be equipped with any MOEA. In this paper, we will apply NSGA-II~\cite{nsgaii}, MOEA/D~\cite{moead} and GSEMO~\cite{giel2003expected,Laumanns04,neumann2006minimum}. NSGA-II is a representative of Pareto dominace based MOEAs, which selects two parents by binary tournament selection, employs crossover and mutation to generate offspring solutions, and updates the population based on non-dominated sorting and crowding distance. MOEA/D is a representative of decomposition based MOEAs, which decomposes a multi-objective optimization problem into a number of scalar optimization subproblems and optimizes them collaboratively. For each current solution, it randomly selects two neighboring solutions as parents, and employs crossover and mutation to generate an offspring solution. By comparing the newly generated offspring solution with neighboring solutions, the better one will be kept in the population. 


GSEMO is relatively simple but has shown good theoretical properties in solving many problems~\cite{neumann2010bioinspired,qian19el,doerr-20-book}. As presented in Algorithm~\ref{alg:GSEMO}, it starts from the all-0 vector $\bm{0}$ (i.e., the empty set) in line~1, and iteratively improves the quality of solutions in the population $P$ (lines~2--7). In each iteration, it first selects a parent solution $\bm{x}$ from the current population $P$ uniformly at random in line~3. Next, it applies bit-wise mutation on $\bm{x}$ (lines~4), which flips each bit of $\bm x$ with probability $1/n$ to generate an offspring solution $\bm{x}'$. Then, $\bm{x}'$ is used to update the population $P$ (lines~5--7). If $\bm{x}'$ is not dominated by any solution in $P$ (line~5), it will be added into $P$, and meanwhile, those solutions weakly dominated by $\bm{x}'$ will be deleted (line~6). By the updating procedure in lines~5--7 of Algorithm~\ref{alg:GSEMO}, we know that the solutions in the population $P$ must be incomparable. Furthermore, $P$ will always contain the initial solution $\bm 0$, because $\bm 0$ has the largest value (i.e., $n$) of $f_2$, and no other solution can weakly dominate it. Thus, any infeasible solution $\bm x$ will be excluded from $P$, because $f_1(\bm x)=-1$ and $f_2(\bm x)<n$, implying that $\bm x$ is dominated by $\bm 0$, whose two objective values are 0 and $n$, respectively. Note that the original version of GSEMO in~\cite{Laumanns04} starts from a random initial solution, while we modify it to start from the all-0s solution $\bm{0}$, as the all-0s solution will be used in theoretical analysis, which will be clear in the next section.

\begin{algorithm}[t!]
\caption{GSEMO Algorithm}
\label{alg:GSEMO}
\textbf{Process}:
\begin{algorithmic}[1]
\STATE Let $P=\{\bm{0}\}$;
\STATE \textbf{repeat}
\STATE \quad Select $\bm{x}$ from $P$ uniformly at random;
\STATE  \quad Apply bit-wise mutation on $\bm{x}$ to generate $\bm{x}'$;
\STATE \quad\textbf{if}  {$\nexists \bm z \in P$ such that $\bm z \succ \bm{x}'$} \,\textbf{then}
    \STATE  \quad\quad $P \gets (P \setminus \{\bm z \in P \mid \bm{x}' \succeq \bm z\}) \cup \{\bm{x}'\}$
\STATE \quad\textbf{end if}
\STATE \textbf{until} some criterion is met
\end{algorithmic}
\end{algorithm}

Our experimental results will show that MR-EMO using the existing NSGA-II, MOEA/D or GSEMO has already performed better than previous algorithms. Next, by analyzing the characteristics of the migrant resettlement problem and the limitation of the traditional reproduction operators, we will develop a new MOEA, which is more suitable for MR-EMO and will lead to better performance.

\subsection{GSEMO-SR Algorithm}

As observed in Figure~\ref{heatmap}(b), there are a large number of infeasible solutions for the migration resettlement problem. In the previous section, we have represented a subset $X$ of $N$ migrant-locality pairs (i.e., $X\subseteq N= V\times L$, where $V$ and $L$ denote the set of migrants and localities, respectively) by a Boolean vector $\bm{x}\in\{0,1\}^n$, where $n=|N|=|V|\times|L|$. If using the traditional one-point crossover and bit-wise mutation operators, the generated offspring solutions are easily infeasible, and thus the feasible regions cannot be explored efficiently. To better illustrate this phenomenon, we also represent a solution $\bm{x}\in\{0,1\}^{n}$ by a matrix of size $|V|\times|L|$, where each row and column correspond to one migrant $v_i$ and one locality $l_j$, respectively, and $x_{\{v_i,l_j\}}$ denotes the element at the $i$-th row and $j$-th column of the matrix, corresponding to the $((i-1)\cdot |L|+j)$-th bit of $\bm{x}$. Take a migrant resettlement problem instance with six migrants and three localities as an example, where each capacity of the locality is two. Figure~\ref{solution} gives a Boolean vector solution of this instance, as well as its corresponding matrix representation, where migrants $\{v_2,v_6\}$, $\{v_1,v_5\}$ and $\{v_3,v_4\}$ are assigned to locality $l_1$, $l_2$ and $l_3$, respectively. This solution is feasible because the number of $1$-bits in each row does not exceed 1 and the number of $1$-bits in each column does not exceed 2, i.e., the capacity of each locality. That is, each migrant is distributed to at most one locality and each locality accepts at most two migrants; thus, the solution satisfies the two matroid constraints. Note that by the matrix representation, the behaviors of one-point crossover and bit-wise mutation can still be illustrated clearly.

\begin{figure}[t!]
\centering
\includegraphics[width=0.6\columnwidth]{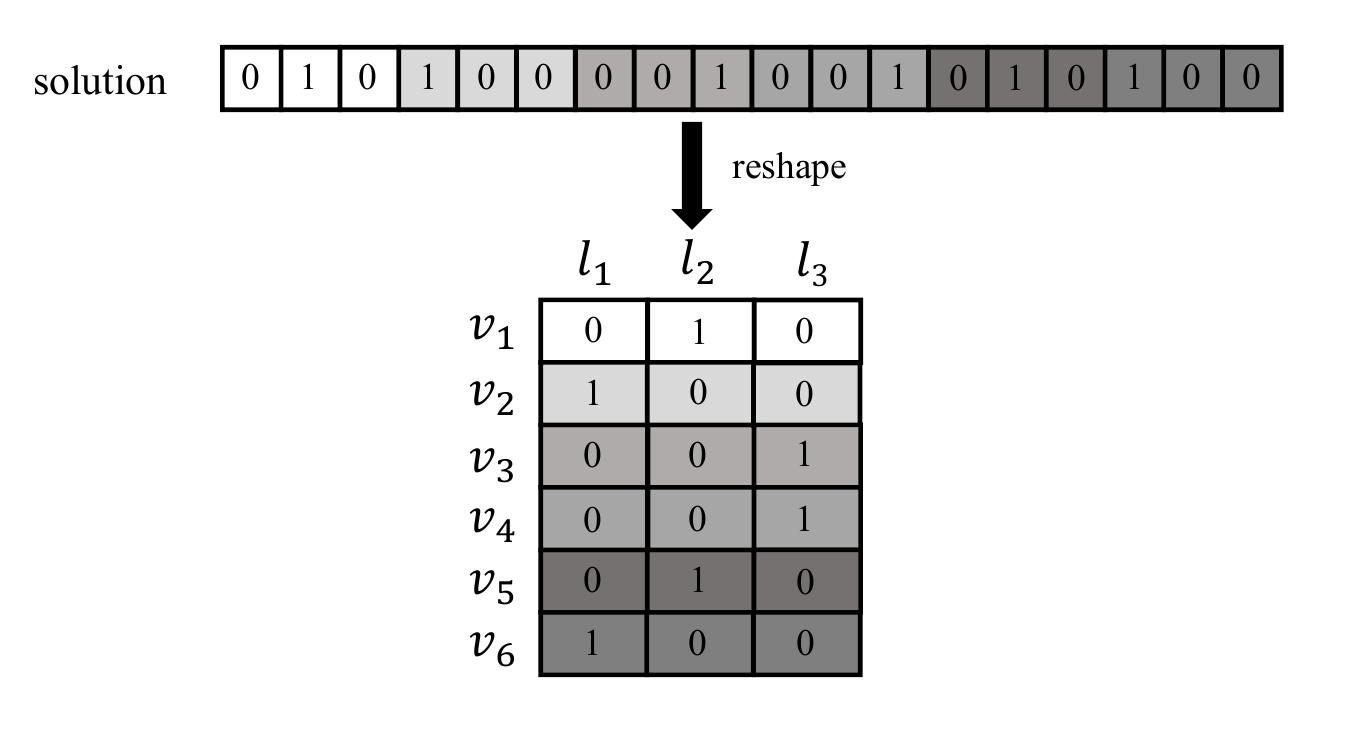}
\caption{A feasible solution (Boolean vector representation and matrix representation) of a migrant resettlement problem instance with six migrants and three localities, where each locality can accommodate at most two migrants.}
\label{solution}
\end{figure}
One-point crossover selects $i\in\{1,2,\cdots, n\}$ randomly, and exchanges the bits after the $i$-th bit of two parent solutions. For Figure~\ref{crossover_mutation}(a), the one-point crossover operator selects $i=10$ and then exchanges the bits after the $10$-th bit of $\bm{x}$ and $\bm{y}$ to generate offspring solutions $\bm{x}'$ and $\bm{y}'$. Bit-wise mutation flips each bit of a solution independently with probability $1/n$. For Figure~\ref{crossover_mutation}(b), the bit-wise mutation flips the $1$-st and $11$-th bits of $\bm{x}$ and keeps other bits unchanged to generate an offspring solution $\bm{x}'$. In the above two examples, we can find that the parents are all feasible, but produce infeasible offspring solutions. In fact, this phenomenon will often occur by applying traditional one-point crossover and bit-wise mutation. For parent solutions $\bm{x}$ and $\bm{y}$ in Figure~\ref{crossover_mutation}(a), the results of executing one-point crossover on $\bm{x}$ and $\bm{y}$ all lead to two infeasible offspring solutions, except for exchanging the bits after the $1$-st, $17$-th or $18$-th bit of them. For parent solution $\bm{x}$ in Figure~\ref{crossover_mutation}(b), an extra $1$-bit in any row or column of $\bm{x}$ will result in an infeasible solution $\bm{x}'$. Thus, the one-point crossover and bit-wise mutation operators may not explore the feasible solution space efficiently. 

\begin{figure*}[t!]\centering
\begin{minipage}[c]{0.53\linewidth}\centering
        \includegraphics[width=0.9\linewidth]{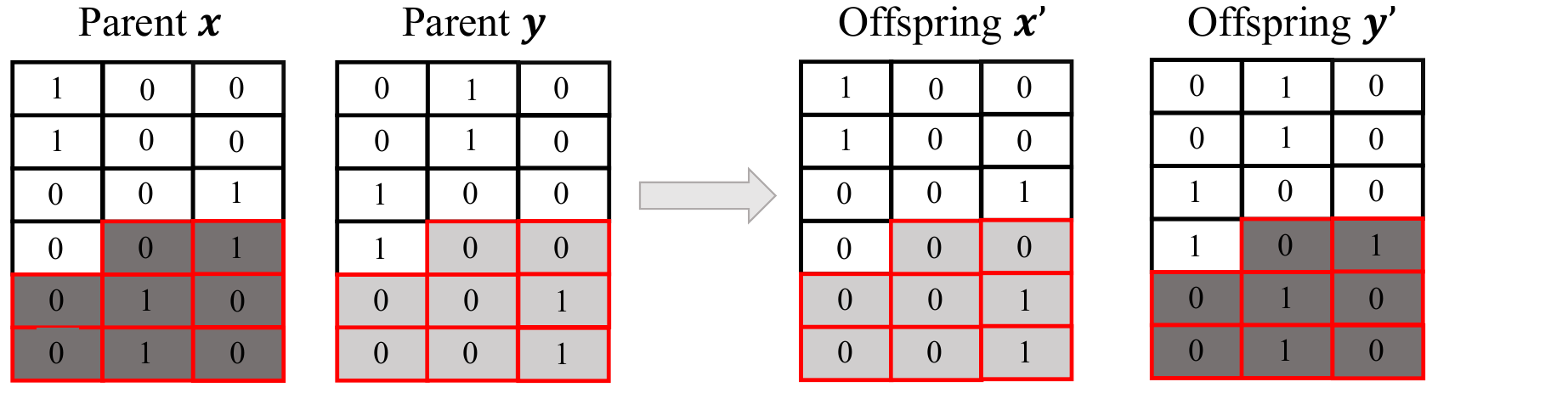}
\end{minipage}
\begin{minipage}[c]{0.3\linewidth}\centering
        \includegraphics[width=0.9\linewidth]{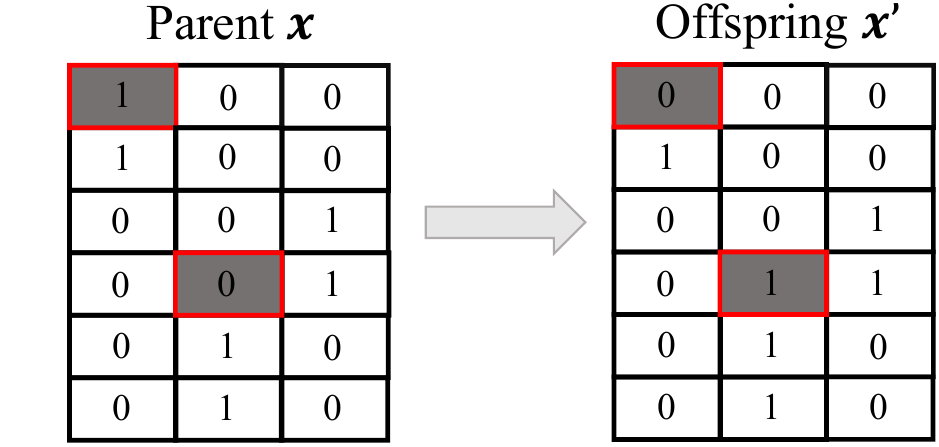}
\end{minipage}\\\vspace{1em}
\begin{minipage}[c]{0.53\linewidth}\centering
    \small(a) One-point crossover
\end{minipage}
\begin{minipage}[c]{0.3\linewidth}\centering
    \small(b) Bit-wise mutation
\end{minipage}
\caption{Examples illustration of (a) one-point crossover and (b) bit-wise mutation.}\label{crossover_mutation}
\end{figure*}

\begin{figure*}[t!]
\begin{minipage}[c]{0.3\linewidth}\centering
        \includegraphics[width=0.8\linewidth]{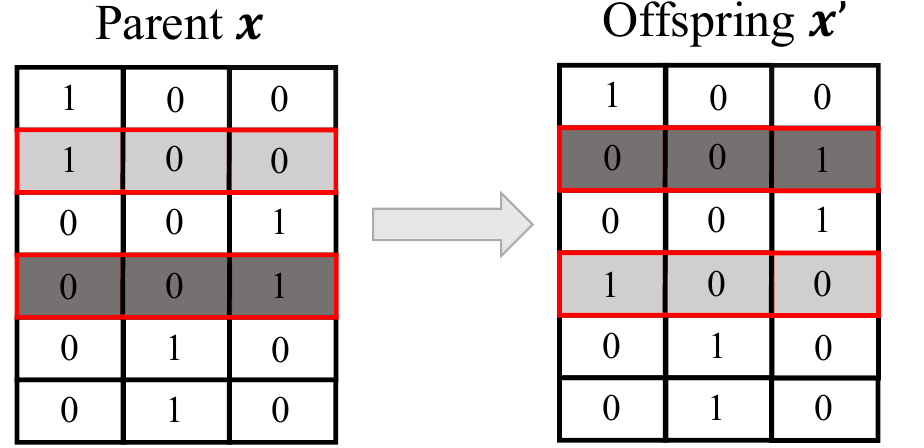}
\end{minipage}
\begin{minipage}[c]{0.3\linewidth}\centering
        \includegraphics[width=0.8\linewidth]{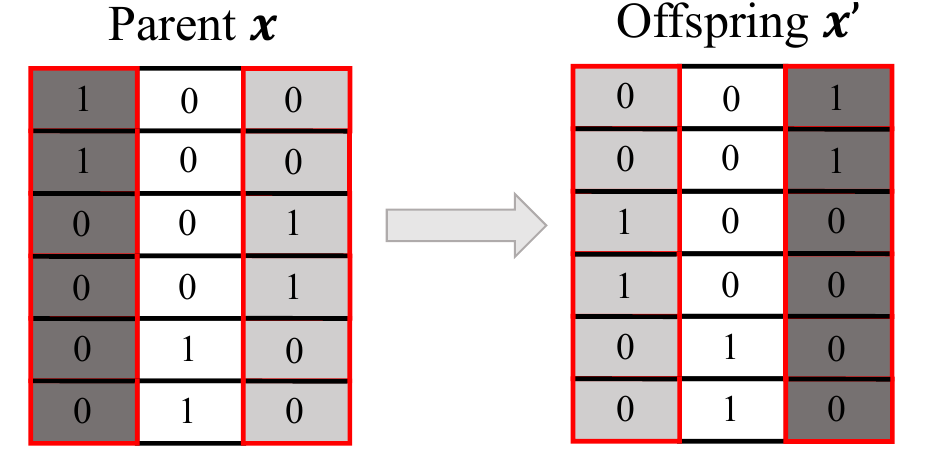}
\end{minipage}
\begin{minipage}[c]{0.3\linewidth}\centering
        \includegraphics[width=0.8\linewidth]{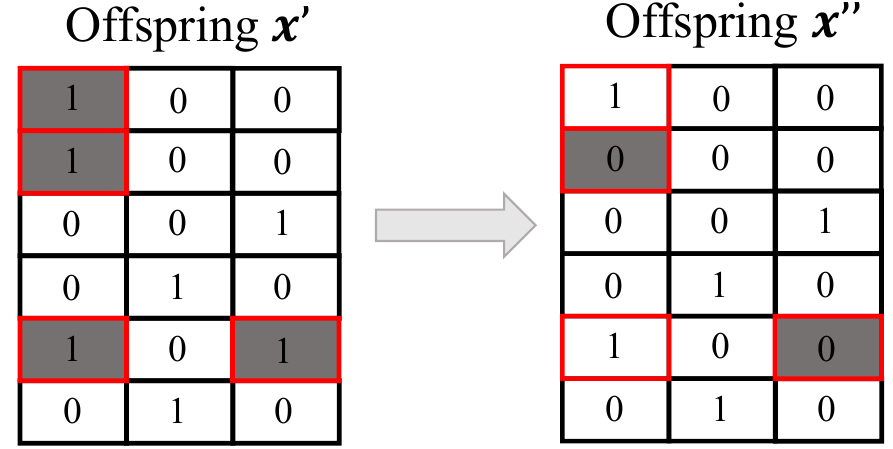}
\end{minipage}\vspace{1em}
\begin{minipage}[c]{0.34\linewidth}\centering
    \small(a) Matrix-swap mutation \\by rows
\end{minipage}
\begin{minipage}[c]{0.3\linewidth}\centering
    \small(b) Matrix-swap mutation \\by columns
\end{minipage}
\begin{minipage}[c]{0.3\linewidth}\centering
    \small(c) Repair mechanism
\end{minipage}
\caption{Example illustration of (a) matrix-swap mutation by rows and (b) matrix-swap mutation by columns, and (c) repairing an infeasible solution.}\label{matrix_swap_repair}
\end{figure*}

We then propose a new mutation operator named matrix-swap mutation in Definition~\ref{def-matrix-swap}, which has a better ability to explore the feasible region of solution space. It can be performed by rows or columns. Figure~\ref{matrix_swap_repair}(a) shows that matrix-swap mutation by rows selects the second and the fourth rows uniformly at random, and exchanges the corresponding bits of the second and the fourth rows of $\bm{x}$ to generate a new offspring solution $\bm{x}'$. The matrix-swap mutation ensures that $\bm{x}'$ is still feasible because it has the same number of $1$-bits of each row and column as its feasible parent $\bm{x}$. Also in Figure~\ref{matrix_swap_repair}(b), matrix-swap mutation by columns randomly selects the first and the third columns, and then $\bm{x}'$ is generated by swapping the corresponding bits of the first and the third columns of $\bm{x}$. Note that $\bm{x}'$ is also feasible if the capacity of each locality is the same. Compared with classical operators (i.e., the one-point crossover and bit-wise mutation operators), the matrix-swap mutation operator can better explore the solution space while maintaining solution feasibility.

\begin{definition}[Matrix-swap mutation]\label{def-matrix-swap}
Let $M^{|V|\times|L|}$ be a matrix representation of solution $\bm x\in\{0,1\}^{n}$, where $n=|V|\times|L|$. The matrix-swap mutation by rows selects $i,j\in\{1,\cdots,|V|\}$ uniformly at random with replacement, and then exchanges the $i$-th row and the $j$-th row of $M$. Similarly, matrix-swap mutation by columns uniformly selects $i,j\in\{1,\cdots,|L|\}$ at random with replacement and exchanges the 
$i$-th column and the $j$-th column of $M$.
\end{definition}

As observed in Figure~\ref{heatmap}(b), good feasible solutions may be surrounded by infeasible ones. Thus, we introduce a repair mechanism in Definition~\ref{def-repair} to repair an infeasible offspring solution to be feasible, instead of discarding it directly. For an offspring $\bm{x}'$ in Figure~\ref{matrix_swap_repair}(c), the number of $1$-bits of the fifth row of $\bm{x}'$ is 2, and then the repair mechanism randomly flips $1=\max\{2-1,0\}$ $1$-bit of the fifth row. Similarly, the total number of $1$-bits of the first column is 3, which exceeds the capacity 2 of the first locality, and $1=\max\{3-2,0\}$ $1$-bit of the first column has been flipped. After repairing each row and each column of $\bm{x}'$ in this way, the new offspring $\bm{x}''$ must be feasible. Note that the repair mechanism will not change any bit of a feasible offspring solution.

\begin{definition}[Repair mechanism]\label{def-repair}
Let $M^{|V|\times|L|}$ be a matrix representation of solution $\bm x\in\{0,1\}^{n}$, where $n=|V|\times|L|$. The repair mechanism first counts the total number of $1$-bits of each row $i$ from M, denoted by $x_{\{v_i,\cdot\}}=\sum_{j=1}^{|L|} x_{\{v_{i},l_{j}\}}$. Then it randomly flips $\max\{x_{\{v_{i},\cdot\}}-1,0\}$ $1$-bits of row $i$ and keeps other $0$-bits unchanged. Next, it also counts the total number of $1$-bits of each column $j$, denoted by $x_{\{\cdot,l_j\}}=\sum_{i=1}^{|V|} x_{\{v_i,l_j\}}$, and randomly flips $\max\{x_{\{\cdot,l_j\}}-cap_{l_j},0\}$ $1$-bits and keeps other $0$-bits unchanged.
\end{definition}


By incorporating the matrix-Swap mutation operator and the Repair mechanism into GSEMO, we propose a new MOEA, called GSEMO-SR. As presented in Algorithm~\ref{alg:GSEMO-SR}, it starts from the all-0 vector $\bm{0}$ (i.e., the empty set) in line~1, and iteratively improves the quality of solutions in the population $P$ (lines~2--13). In each iteration, it first selects a parent solution $\bm{x}$ from the current population $P$ uniformly at random in line~3. Next, it applies bit-wise mutation on $\bm{x}$ with probability $p_m$ (lines~5--6), which flips each bit of $\bm x$ independently with probability $1/n$ to generate an offspring solution $\bm{x}'$. Otherwise, GSEMO-SR applies the proposed matrix-swap mutation operator on $\bm{x}$ (lines~7--8), which can be performed by rows or columns with equal probability. Then the offspring solution $\bm{x}'$ will go through a repair mechanism in line~10, which can repair an infeasible solution to be feasible. Note that $\bm{x}''$ equals to $\bm{x}'$ if $\bm{x}'$ has already been feasible, and GSEMO-SR always contains feasible solutions. The repaired offspring solution, $\bm{x}''$ is used to update the population $P$ (lines~11--13). If $\bm{x}''$ is not dominated by any solution in $P$ (line~11), it will be added into $P$, and meanwhile, those solutions weakly dominated by $\bm{x}''$ will be deleted (line~12). Similar to GSEMO, the solutions contained by the population $P$ of GSEMO-SR must be incomparable.

\begin{algorithm}[ht]
\caption{GSEMO-SR Algorithm}
\label{alg:GSEMO-SR}
\textbf{Parameter}: probability $p_m\in\{0,1\}$\\
\textbf{Process}:
\begin{algorithmic}[1]
\STATE Let $P=\{\bm{0}\}$;
\STATE \textbf{repeat}
\STATE \quad Select $\bm{x}$ from $P$ uniformly at random;
\STATE \quad $q \leftarrow$ a number selected from $[0,1]$ uniformly at random;\\
\STATE \quad\textbf{if}  {$q\leq p_m $} \,\textbf{then}
\STATE \quad \quad Apply bit-wise mutation on $\bm{x}$ to generate $\bm{x}'$
\STATE \quad\textbf{else}
\STATE \quad \quad Apply matrix-swap mutation by rows or columns
\STATE \quad \quad (each with probability $1/2$) on $\bm{x}$ to generate $\bm{x}'$
\STATE \quad\textbf{end if}
\STATE \quad $\bm{x}''\leftarrow$ Repair($\bm{x}'$);
\STATE \quad\textbf{if}  {$\nexists \bm z \in P$ such that $\bm z \succ \bm{x}''$} \,\textbf{then}
    \STATE  \quad\quad $P \gets (P \setminus \{\bm z \in P \mid \bm{x}'' \succeq \bm z\}) \cup \{\bm{x}''\}$
\STATE \quad\textbf{end if}
\STATE \textbf{until} some criterion is met
\end{algorithmic}
\end{algorithm}

Compared with GSEMO-SR, GSEMO in Algorithm~\ref{alg:GSEMO} only applies bit-wise mutation to generate a new offspring solution and no repair mechanism. In the experiments, we will show that using GSEMO-SR leads to the best performance of MR-EMO, which is significantly better than using NSGA-II, MOEA/D or GSEMO. We will also perform ablation study to validate the effectiveness of the introduced matrix-swap mutation operator and repair mechanism. Note that the matrix-swap mutation operator and repair mechanism can also be combined with other MOEAs to bring performance improvement, which will be studied in the future.

\section{Theoretical Analysis}\label{Theo}

Before examining the empirical performance of the proposed MR-EMO framework, we show that MR-EMO can achieve better theoretical guarantees than the previous greedy algorithm~\cite{golz2019migration} in this section.

First, we prove in Theorem~\ref{main-theorem} that MR-EMO using GSEMO (briefly called MR-GSEMO) can achieve an approximation guarantee of $1/(k+\frac{1}{p}+\frac{2\epsilon r }{1-\epsilon}+\frac{(1+\epsilon)\delta r}{1-\epsilon})$ after running at most $O(\frac{rn^{(k+1)p}}{\delta}\log{\frac{(1+\epsilon)r}{1-\epsilon}})$ expected number of iterations, where $k$ is the number of matroids, $\epsilon$ is the degree of approximate submodularity of $\hat{f}$ as shown in Definition~\ref{def-epsilon-app-submodular}, and $r$ is the size of the largest feasible solution. Note that $\mathrm{OPT}$ denotes the optimal function value of Eq.~(\refeq{eq:subsetsel}).

\begin{theorem}\label{main-theorem}
MR-GSEMO using $ O\left(\frac{rn^{2(k+1)p}}{\delta}\log{\frac{(1+\epsilon)r}{1-\epsilon}}\right)$ expected number of iterations finds a solution $X\in \bigcap_{i=1}^{k}\mathcal{F}_{i}$ with
\begin{equation}
\begin{aligned}\label{eq-approximation1}
\left(k+\frac{1}{p}+\frac{2\epsilon r }{1-\epsilon}+\frac{(1+\epsilon)\delta r}{1-\epsilon} \right)\cdot f(X)\ge \mathrm{OPT},
\end{aligned}
\end{equation}
where $k\ge2$, $p\ge1$, $\epsilon\ge0$ and $\delta>0$.
\end{theorem}

The main proof idea is to show that MR-GSEMO can find a $(1+\delta)$-approximate $p$-local optimal solution in Definition~\ref{def-locally-optimal}, which guarantees the desired approximation ratio as shown in Lemma~\ref{lemma-local-optimal}.

\begin{definition}[$(1+\delta)$-approximate $p$-local optimum]\label{def-locally-optimal}
A $(1+\delta)$-approximate $p$-local optimal solution $X$ is a feasible solution such that no feasible solution $X'$ with the objective value $\hat{f}(X')>(1+\delta)\cdot \hat{f}(X)$ can be achieved by deleting at most $2kp$ elements from $X$ and adding at most $2p$ new elements, where $\delta>0$ and $p\geq 1$.
\end{definition}

\begin{lemma}\label{lemma-local-optimal}
Any $(1+\delta)$-approximate $p$-local optimal solution $X$ satisfies 
\begin{equation}
\begin{aligned}
\left(k+\frac{1}{p}+\frac{2\epsilon r }{1-\epsilon}+\frac{(1+\epsilon)\delta r}{1-\epsilon} \right)\cdot f(X)\ge \mathrm{OPT},
\end{aligned}
\end{equation}
i.e., Eq.~(\refeq{eq-approximation1}), where $k\ge2$, $p\ge1$, $\epsilon\ge0$ and $\delta>0$.
\end{lemma}

The proof of Lemma~\ref{lemma-local-optimal} will use the following two lemmas.

\begin{lemma}[Lemma~1.1 in~\cite{jonlee10}]\label{jonlee-lemma1} Let $f$ be a submodular function. For any $X,C\subseteq N$, let $\{T_{j}\}_{j=1}^{t}$ be a collection of subsets of $C\setminus X$ such that each element of $C\setminus X$ appears in exactly $k$ of these subsets. Then, $\sum_{j=1}^{t}(f(X\cup T_j)-f(X))\ge k(f(X\cup C)-f(X))$.
\end{lemma}

\begin{lemma}[Lemma~1.2 in~\cite{jonlee10}]\label{jonlee-lemma2} Let $f$ be a submodular function. For any $X'\subseteq X\subseteq N$, let $\{T_{j}\}_{j=1}^{t}$ be a collection of subsets of $X\setminus X'$ such that each element of $X\setminus X'$ appears in exactly $k$ of these subsets. Then, $\sum_{j=1}^{t}(f(X)-f(X\setminus T_j))\leq k(f(X)-f(X'))$.
\end{lemma}

\begin{myproof}{Lemma~\ref{lemma-local-optimal}}
Let $X$ denote a $(1+\delta)$-approximate $p$-local optimal solution, and $C$ denote any other feasible solution. By reusing the construction process in Lemma~3.1 of~\cite{jonlee10}, there are two collections of subsets $\{W_j\}_{j=1}^{t}$ of $C\setminus X$ and $\{\Lambda_j\}_{j=1}^{t}$ of $X\setminus C$ for $t \leq p 2^q|C|$ and $q\ge 0$, such that each element of $C\setminus X$ appears in exactly $p2^q$ of sets $\{W_j\}_{j=1}^{t}$ and each element $e\in X\setminus C$ appears in exactly $n_e$ of sets $\{\Lambda_j\}_{j=1}^{t}$, where $n_e\leq(k+1/p-1)\cdot p2^q$. Furthermore, $|W_j|\leq 2p$, $|\Lambda_j|\leq 2kp$, and the set $(X\setminus\Lambda_j) \cup W_j$ is feasible. As $X$ is $(1+\delta)$-approximately $p$-local optimal, we have 
\begin{equation}
\begin{aligned}
\hat{f}((X\setminus\Lambda_j)\cup W_j) \leq (1+\delta)\cdot \hat{f}(X)
\end{aligned}
\end{equation}
according to Definition~\ref{def-locally-optimal}. Using the relation between $\hat{f}$ and $f$ in Eq.~(\refeq{eq:epsilon}), we can further get
\begin{align}\label{eq:inequality1}
f((X\setminus\Lambda_j)\cup W_j)\leq \frac{1+\epsilon}{1-\epsilon}(1+\delta)\cdot f(X).
\end{align}

By the submodularity of $f$, i.e., Eq.~(\refeq{eq-submodular-1}), we have
\begin{align}\label{eq:inequality3}
    &\ \sum_{j=1}^{t}\left(f(X\cup W_j)-f(X)\right)\\
    &\leq \sum_{j=1}^{t} \left(f((X \setminus\Lambda_j)\cup W_j)-f(X \setminus\Lambda_j)\right)\nonumber\\
    &\leq \sum_{j=1}^{t}\left(\frac{1+\epsilon}{1-\epsilon}(1+\delta)\cdot f(X)-f(X \setminus\Lambda_j)\right)\nonumber\\
        &\leq \sum_{j=1}^{t}\left(\frac{1+\epsilon}{1-\epsilon}(1+\delta)\cdot f(X)-f(X \setminus\Lambda_j)\right)+\sum_{e\in X\setminus C} ((k+1/p-1)\cdot p2^q-n_e)(f(X)-f(X\setminus\{e\}))\nonumber\\
    &=t\left(\frac{1+\epsilon}{1-\epsilon}(1+\delta)-1\right) f(X)+\sum_{j=1}^{\lambda}(f(X)-f(X\setminus\Gamma_j)),\nonumber
\end{align}
where the second inequality holds by Eq.~(\refeq{eq:inequality1}), the third inequality holds by $n_e\leq (k+1/p-1)\cdot p2^q$ and $f(X)\ge f(X\setminus\{e\})$ due to the monotonicity of $f$, and the equality holds by letting $\lambda=t+\sum_{e\in X\setminus C}((k+1/p-1)\cdot p2^q-n_e)$, $\forall 1\leq j\leq t: \Gamma_j=\Lambda_j$, and $\{\Gamma_j\}^{\lambda}_{j=t+1}$ be a collection of singleton sets $\{e\}$ where for each $e\in X\setminus C$, $\{e\}$ appears $(k+1/p-1)\cdot p2^q-n_e$ times.

As $\{W_j\}_{j=1}^{t}$ is a collection of subsets of $C\setminus X$ such that each element of $C\setminus X$ appears in exactly $p2^q$ of these subsets, applying Lemma~\ref{jonlee-lemma1} to the left side of Eq.~(\refeq{eq:inequality3}) leads to
\begin{align}\label{eq-mid-1}
\sum\limits_{j=1}^{t}\left(f(X\cup W_j)\!-\!f(X)\right)\!\geq\! p2^q\cdot(f(X\cup C)\!-\!f(X)).
\end{align}
We know that $\{\Gamma_j\}_{j=1}^{\lambda}$ is a collection of subsets of $X\setminus C$, and each element $e \in X\setminus C$ appears in exactly $(k+1/p-1)\cdot p2^q$ of these subsets (i.e., $n_e$ times in $\{\Gamma_j\}^t_{j=1}=\{\Lambda_j\}^t_{j=1}$ and $(k+1/p-1)\cdot p2^q-n_e$ times in $\{\Gamma_j\}^{\lambda}_{j=t+1}$). Thus, applying Lemma~\ref{jonlee-lemma2} to the last term of Eq.~(\refeq{eq:inequality3}) leads to
\begin{align}\label{eq-mid-2}
\sum_{j=1}^{\lambda}(f(X)-f(X\setminus\Gamma_j))
\leq (k+1/p-1)\cdot p2^q\cdot (f(X)-f(X\cap C)).
\end{align}
As the size of a feasible solution is at most $r$, we have $t\leq p 2^q|C| \leq rp2^q$. By applying this inequality, Eqs.~(\refeq{eq-mid-1}) and~(\refeq{eq-mid-2}) to Eq.~(\refeq{eq:inequality3}), we have
\begin{align*}
p2^q(f(X\cup C)-f(X))\leq r p2^q\left(\frac{1+\epsilon}{1-\epsilon}(1+\delta)-1\right)f(X)+(k+1/p-1)p2^q\cdot (f(X)-f(X\cap C)).
\end{align*}
By simple calculation, we have
\begin{align}\label{eq-mid-3}
\left(k+\frac{1}{p}+\frac{2\epsilon r }{1-\epsilon}+\frac{(1+\epsilon)\delta r}{1-\epsilon}\right)\cdot f(X)\ge f(X\cup C)+(k+1/p-1)\cdot f(X\cap C).
\end{align}

Let $C$ be an optimal solution of Eq.~(\refeq{eq:subsetsel}), i.e., $f(C)=\mathrm{OPT}$. By the monotonicity of $f$, $f(X\cup C)\ge f(C)= \mathrm{OPT}$, and $f(X\cap C)\ge f(\emptyset)= 0$. Thus, Eq.~(\refeq{eq-mid-3}) becomes
\begin{align*}
\left(k+\frac{1}{p}+\frac{2\epsilon r }{1-\epsilon}+\frac{(1+\epsilon)\delta r}{1-\epsilon}\right)\cdot f(X) \geq \mathrm{OPT},
\end{align*}
implying that the lemma holds.
\end{myproof}

Now we are ready to prove Theorem~\ref{main-theorem}. Note that as presented in Algorithm~\ref{alg:GSEMO}, GSEMO only uses the bit-wise mutation operator to generate a new offspring solution in each iteration, and uses this offspring solution to update the population such that only the non-dominated solutions generated so far are kept in the population.

\begin{myproof}{Theorem~\ref{main-theorem}}
We divide the optimization process of MR-GSEMO into two phases: (1) starts from the initial solution $\bm{0}$ (i.e., $\emptyset$) and finishes after finding a feasible solution $\bm{x}$ with the objective value $\hat{f}(\bm{x})\ge ((1-\epsilon)/r)\cdot \mathrm{OPT}$; (2) starts after phase (1) and finishes after finding a feasible solution with the desired approximation ratio in Eq.~(\refeq{eq-approximation1}). Note that the population $P$ will always contain feasible solutions, because infeasible solutions are dominated by the initial solution $\bm{0}$.

Let $\bm{x}^*$ denote an optimal solution, i.e., $f(\bm{x}^*)=\mathrm{OPT}$. Let $\{(v,l)^*\}= \mathop{\arg\max}\nolimits_{\{(v,l)\}\in \bigcap_{i=1}^{k}\mathcal{F}_{i}} \hat{f}(\{(v,l)\})$. By the submodularity of $f$ in Eq.~(\refeq{eq-submodular-2}) and $f(\emptyset)=0$, we have
\begin{align*}
    \mathrm{OPT}=f(\bm{x}^*)-f(\emptyset)\leq\sum_{(v,l)\in \bm{x}^*}(f(\{(v,l)\})-f(\emptyset)) 
    \leq \sum_{(v,l)\in \bm{x}^*} \frac{\hat{f}(\{(v,l)\})}{1-\epsilon} \leq r\cdot \frac{\hat{f}(\{(v,l)^*\})}{1-\epsilon},
\end{align*}
where the second inequality holds by Eq.~(\ref{eq:epsilon}), and the last inequality holds by the definition of $\{(v,l)^*\}$ and $|\bm{x}^*| \leq r$, i.e., the largest size of a feasible solution. Thus, 
\begin{align*}
    \hat{f}(\{(v,l)^*\})\ge ((1-\epsilon)/r)\cdot \mathrm{OPT}.
\end{align*}   
Note that the initial solution $\bm{0}$ will always exist in the population $P$, because it has the largest $f_2$ value and no other solution can dominate it. By selecting $\bm{0}$ in line~3 of Algorithm~\ref{alg:GSEMO} and flipping only one specific 0-bit (corresponding to the migrant-locality pair $(v,l)^*$) by bit-wise mutation in line~4, the solution $\{(v,l)^*\}$ will be generated, which will then always exist in $P$, because it is the best solution with the $f_2$ value of $n-1$ and no other solution can dominate it. This implies that the goal of phase~(1) has been reached. Because the probability of selecting $\bm 0$ in line~3 is $1/|P|$ due to uniform selection, and the probability of flipping a specific 0-bit while keeping other bits unchanged in bit-wise mutation is $(1/n) \cdot (1-1/n)^{n-1}\geq 1/(en)$, the expected number of iterations of phase~(1) is at most $en|P| \leq e(r+1)n$. The population size $|P| \leq r+1$ holds, because MR-GSEMO excludes infeasible solutions from the population $P$, and $P$ can contain at most one feasible solution for each size in $\{0,1\ldots,r\}$ due to the incomparability of solutions in $P$.

Lemma~\ref{lemma-local-optimal} shows that a $(1+\delta)$-approximate $p$-local optimal solution achieves the desired approximation ratio in Eq.~(\refeq{eq-approximation1}). Thus, for phase~(2), we only need to analyze the expected number of iterations until generating a $(1+\delta)$-approximate $p$-local optimal solution. We consider the largest $\hat{f}$ (i.e., $f_1$) value of the solutions in $P$, denoted by $F_{\max}$. That is, $F_{\max}=\max\{\hat{f}(\bm{x}) \mid \bm{x} \in P\}$. After phase~(1), $F_{\max}\geq ((1-\epsilon)/r) \cdot \mathrm{OPT}$, and let $\bm{x}$ be the corresponding solution, i.e., $\hat{f}(\bm{x})=F_{\max}$. Obviously, $F_{\max}$ cannot decrease, because $\bm{x}$ cannot be weakly dominated by a solution with a smaller $\hat{f}$ value. As long as $\bm{x}$ is not $(1+\delta)$-approximately $p$-local optimal, we know that a new solution $\bm{x}'$ with $\hat{f}(\bm{x}')>(1+\delta)\cdot \hat{f}(\bm{x})=(1+\delta)\cdot F_{\max}$ can be generated through selecting $\bm{x}$ in line~3 of Algorithm~\ref{alg:GSEMO}, and flipping at most $2(k+1)p$ bits (i.e., deleting at most $2kp$ elements from $\bm{x}$ and adding at most $2p$ new elements) in line~4, whose probability is at least $(1/|P|)\cdot (1/n)^{2(k+1)p}(1-1/n)^{n-2(k+1)p}\geq 1/(e(r+1)n^{2(k+1)p})$. Since $\bm{x}'$ now has the largest $\hat{f}$ value and no solution in $P$ can dominate it, it will be included into $P$. Thus, $F_{\max}$ can increase by at least a factor of $1+\delta$ with probability at least $1/(e(r+1)n^{2(k+1)p})$ in each iteration, implying at most $e(r+1)n^{2(k+1)p}$ expected number of iterations. As the maximum value of $\hat{f}$ is at most $(1+\epsilon)\cdot \mathrm{OPT}$, the required number of such an increase on $F_{\max}$ is at most
\begin{align*}
    &\log_{1+\delta}\frac{(1+\epsilon)\cdot \mathrm{OPT}}{((1-\epsilon)/r)\cdot \mathrm{OPT}}=O\left(\frac{1}{\delta}\log {\frac{(1+\epsilon)r}{1-\epsilon}}\right),
\end{align*}
until generating a $(1+\delta)$-approximate $p$-local optimal solution. Thus, the expected number of iterations of phase~(2) is at most $e(r+1)n^{2(k+1)p} \cdot O(\frac{1}{\delta}\log {\frac{(1+\epsilon)r}{1-\epsilon}}) =  O(\frac{rn^{2(k+1)p}}{\delta}\log {\frac{(1+\epsilon)r}{1-\epsilon}})$, implying that the theorem holds.
\end{myproof}

Thus, we have proven the polynomial-time approximation guarantee, $1/(k+\frac{1}{p}+\frac{2\epsilon r }{1-\epsilon}+\frac{(1+\epsilon)\delta r}{1-\epsilon})$, of MR-GSEMO, which can approximately reach $1/(k+\frac{1}{p}+\frac{2\epsilon r }{1-\epsilon})$ by setting $\delta$ sufficiently small. Note that $p\geq 1$. By comparing it with the previously known guarantee, we have

\begin{remark}
The approximation guarantee, i.e., $1/(k\!+\!\frac{1}{p}\!+\!\frac{2\epsilon r }{1\!-\!\epsilon})$, of MR-GSEMO is better than that, i.e., $1/(k+1+\frac{4\epsilon r }{1-\epsilon})$, of the greedy algorithm~\cite{golz2019migration}.
\end{remark}

Similarly, we can prove that MR-EMO using GSEMO-SR (briefly called MR-GSEMO-SR) can achieve the same approximation guarantee as MR-GSEMO after running $ O\left(\frac{rn^{2(k+1)p}}{\delta\cdot p_m}\log{\frac{(1+\epsilon)r}{1-\epsilon}}\right)$ expected number of iterations, as shown in Theorem~\ref{second-theorem}. Different from GSEMO which always performs bit-wise mutation in each iteration, GSEMO-SR performs bit-wise mutation with probability $p_m$; otherwise, performs matrix-swap mutation. Furthermore, if the generated offspring solution is infeasible, GSEMO-SR will repair it to be feasible.

\begin{theorem}\label{second-theorem}
MR-GSEMO-SR using $ O(\frac{rn^{2(k+1)p}}{\delta p_m} \log{\frac{(1+\epsilon)r}{1-\epsilon}})$ expected number of iterations finds a solution $X\in \bigcap_{i=1}^{k}\mathcal{F}_{i}$ with
\begin{equation}
\begin{aligned}
\left(k+\frac{1}{p}+\frac{2\epsilon r }{1-\epsilon}+\frac{(1+\epsilon)\delta r}{1-\epsilon} \right)\cdot f(X)\ge \mathrm{OPT},
\end{aligned}
\end{equation}
where $k\ge2$, $p\ge1$, $\epsilon\ge0$ and $\delta>0$.
\end{theorem}
\begin{proof}
The proof can be finished by directly following that of Theorem~\ref{main-theorem}. The only difference is that the probability of applying bit-wise mutation here is $p_m<1$ rather than 1, and thus the expected number of iterations of phases (1) and (2) is at most $e(r+1)n/p_m$ and $ O\left(\frac{rn^{2(k+1)p}}{\delta\cdot p_m}\log{\frac{(1+\epsilon)r}{1-\epsilon}}\right)$, respectively. Note that the offspring solution $\bm{x}'$ considered in the proof of Theorem~\ref{main-theorem} is always feasible; thus, the repair mechanism employed by GSEMO-SR will not affect the analysis.
\end{proof}

If the GSEMO or GSEMO-SR starts from a random solution, we can first derive the expected number of iterations until finding the all-0s solution $\bm 0$, and then follow the current proof procedure of Theorem~\ref{main-theorem}. For the additional analysis about the expected number of iterations until finding the all-0s solution, Lemma 1 in~\cite{qian2021multiobjective} has already shown that for maximizing any bi-objective pseudo-Boolean problem with $|\bm{x}|_0$ being one objective, the GSEMO (thus also GSEMO-SR) needs at most $O(rn\log n)$ expected number of iterations. The proof considers the solution with the maximum number of 0-bits in the population $P$, and then analyzes the expected number of iterations needed to increase the number of 0-bits of this solution, until generating the all-0s solution $\bm 0$.

When using NSGA-II or MOEA/D, the theoretical analysis will be more challenging. For example, the analysis of GSEMO and GSEMO-SR relies on that if a solution is deleted, there must exist a solution weakly dominating it in the population; however, NSGA-II employs non-dominated sorting and crowding distance to maintain a fixed-size population, which may delete a non-dominated solution. As GSEMO uses mutation only, we also need to consider the influence of crossover for NSGA-II and MOEA/D. Note that there is some recent progress on the theoretical analysis of NSGA-II~\cite{zheng2022first,0001D22,BianQ22,DoerrQ22,doerr2022first} and MOEA/D~\cite{Li2016APT,moeadaaai19,moead21}. Though these works mainly focus on artificial problems, their analysis techniques might be useful for analyzing the approximation guarantee of MR-EMO using NSGA-II or MOEA/D, which will be an interesting future work.


\section{Empirical Study}\label{exp}

In this section, we empirically examine the performance of MR-EMO on two migration models (i.e., interview and coordination), by comparing its five variants (i.e., MR-NSGA-II-100, MR-NSGA-II-2r, MR-MOEA/D, MR-GSEMO and MR-GSEMO-SR) with two previous algorithms, i.e., the additive~\cite{bansak2018improving} and greedy~\cite{golz2019migration} algorithms. Note that MR-NSGA-II-100 and MR-NSGA-II-2r are configurations of MR-EMO using NSGA-II with different population sizes: 100 and 2 times the size of the Pareto front, respectively. The latter size is inspired by the recent work~\cite{zheng2022first}, calculated as $2 \times (r+1) = 2 \times (\sum_{l=1}^{|L|} cap_l + 1)$\footnote{$r$ is the size of the largest feasible subset of migrant-locality pairs, which equals the sum of capacity of all localities.}. MR-MOEA/D, MR-GSEMO, and MR-GSEMO-SR correspond to MR-EMO using MOEA/D, GSEMO, and GSEMO-SR, respectively. For MR-MOEA/D, the population size is set to a commonly used value 100~\cite{moead}. The initial population of NSGA-II-100, NSGA-II-2r and MOEA/D consists of the all-0 vector solution $\bm{0}$ and other randomly generated solutions. All of them apply one-point crossover and bit-wise mutation in each iteration with probability $0.9$ and $1$, respectively. We incorporate an objective normalization into MOEA/D, and perform the Tchbycheff decomposition approach to decompose the bi-objective reformulated migration resettlement problem in Eq.~(\ref{def-CO-BO}) into several scalar optimization subproblems. During the evolutionary process of MR-NSGA-II-100, MR-NSGA-II-2r and MR-MOEA/D, the generated infeasible offspring solutions violating matroid constraints are discarded directly. As non-dominated solutions may be discarded during the population update process of NSGA-II-100, NSGA-II-2r and MOEA/D~\cite{li2023multi}, to ensure a fair comparison, MR-NSGA-II-100, MR-NSGA-II-2r and MR-MOEA/D save the best feasible solution found at each iteration. This solution is returned after the algorithm terminates, even if it has been deleted from the population during the update process. MR-GSEMO starts from the initial solution $\bm{0}$. For MR-GSEMO-SR, the probability $p_m$ of MR-GSEMO-SR applying bit-wise mutation is set to $0.5$. These four variants of MR-EMO are all anytime algorithms, whose performance will be gradually improved by increasing the number of iterations. To make a tradeoff between performance and runtime, we set the number of objective evaluations to $100\cdot |V|^2\cdot |L|$, while the running time complexity of the greedy algorithm~\cite{golz2019migration} is $|V|^2\cdot |L|$.

The experimental data is generated as in~\cite{golz2019migration}. For the interview model, the matching probability $p_{v,l}$ of each migrant $v \in V$ to find employment at each locality $l \in L$ is randomly sampled from $[0,1]$. For the coordination model, each migrant $v\in V$ has matching probability $p_{v,\pi}$, randomly sampled from $[0,1]$, to work in for a job of her profession $\pi \in \Pi$, and is incompatible with the jobs of other professions. We will compare the algorithms by varying the number of migrants, localities, jobs and professions, respectively. For each setting, we generate 10 problem instances randomly, and report the average objective value and the standard deviation, followed by the Wilcoxon signed-rank test~\cite{wilcoxon1945individual} with confidence level $0.05$. Note that to estimate the objective value, i.e., Eqs.~(\refeq{eq-interview-model}) and~(\refeq{eq-coord-model}), of a solution, we use the average of $1,000$ random simulations (which corresponds to the approximate objective $\hat{f}$) in the run of each algorithm, while for the final output solution, we report its accurately estimated $f$ value for the assessment of the algorithms by an expensive evaluation, i.e., the average of $10,000$ random simulations.

In the following, we will introduce the detailed setting (e.g., how to assign each migrant's profession, determine the profession of each job, distribute jobs to each locality, and set the capacity of each locality) of varying the number of migrants $|V|$, localities $|L|$, jobs $|J|$ and professions $|\Pi|$, respectively, followed by the analysis of the experimental results. 

\textbf{Varying the Number $|V|$ of Migrants.} We vary $|V| \in \{100,120,\ldots,$ $200\}$, and fix the number of localities $|L|=10$, the number of jobs $|J|=|V|$ and the number of professions $|\Pi|=2$. The two professions are evenly distributed to the migrants as well as the jobs, i.e., half of the migrants (or jobs) have profession $\pi_1$, and the other half have profession $\pi_2$. The jobs are randomly assigned to the localities such that each locality has the same number of jobs, i.e., $|J|/|L|$. The capacity of each locality is set to its number of jobs.

\textbf{Varying the Number $|L|$ of Localities.} We vary $|L|\in\{16,18,\ldots,30\}$, and fix $|V|=|J|=100$ and $|\Pi|=2$. The other setting is the same as that of varying $|V|$, except that the jobs here are randomly assigned to the localities such that each locality has at least one job.

\textbf{Varying the Number $|J|$ of Jobs.} We fix $|V|=100$, $|L|=10$ and $|\Pi|=2$. To simulate different levels of supply and demand, we vary the number of jobs in $\{10,20,30,40\}\cup\{60,70,80,90\}$ for one profession and fix 50 jobs for the other profession. The professions of migrants are evenly distributed, the jobs are randomly assigned to the localities, and the capacity of each locality is 10.

\textbf{Varying the Number $|\Pi|$ of Professions.} We fix $|V|=|J|=100$, $|L|=10$ and vary $|\Pi|\in\{5,10,\ldots,30\}$. The professions of migrants are assigned randomly, ensuring that each profession has at least one migrant. And there is a job of the right profession for every migrant, that is, the number of jobs of each profession is exactly the same as the number of migrants of the profession. Then, we randomly distribute 10 jobs to each locality and the capacity of each locality is fixed to 10. 

\textbf{Result Comparison.} The detailed results of varying the number of migrants $|V|$, localities $|L|$, jobs $|J|$ and professions $|\Pi|$ are shown in Tables~\ref{tab_V} to~\ref{tab_Pi}, respectively. We can observe that the greedy algorithm is always better than the additive algorithm, which is consistent with the observation in~\cite{golz2019migration}; while the five variants of MR-EMO (i.e., MR-NSGA-II-100, MR-NSGA-II-2r, MR-MOEA/D, MR-GSEMO and MR-GSEMO-SR) all surpass the greedy and additive algorithms, exhibiting the superiority of the MR-EMO framework. This may be because MR-EMO naturally maintains a population of diverse solutions due to the bi-objective transformation, and the employed bit-wise mutation operator has a good global search ability. These characteristics can lead to a better ability to escape from local optima. 

\begin{table*}[th!]\caption{Expected number of employed migrants (mean$\pm$std) obtained by the algorithms when the number of migrants $|V|\in\{100,120,\ldots,200\}$. For each $|V|$, the largest number is bolded, and `$\bullet$' denotes that MR-GSEMO-SR is significantly better than the corresponding algorithm by the Wilcoxon signed-rank test with confidence level $0.05$. Avg.R. denotes the average rank (the smaller, the better) of each algorithm under each setting as in~\cite{DBLP:journals/jmlr/Demsar06}.}\label{tab_V}
\scriptsize
\begin{center}
\setlength{\tabcolsep}{0.75mm}{
\begin{tabular}{c|llllll|c}
\hline\noalign{}
\multicolumn{8}{c}{Interview Model}\\
\hline
$|V|$  & \multicolumn{1}{c}{$100$} &  \multicolumn{1}{c}{$120$} & \multicolumn{1}{c}{$140$} & \multicolumn{1}{c}{$160$} & \multicolumn{1}{c}{$180$} & \multicolumn{1}{c}{$200$}  & Avg.R.\\
\hline
Additive &66.76$\pm$3.10$\bullet$ &82.70$\pm$3.21$\bullet$ &99.83$\pm$2.65$\bullet$ &114.15$\pm$5.62$\bullet$ &135.09$\pm$3.87$\bullet$ &153.29$\pm$4.59$\bullet$ &7\\
Greedy &75.15$\pm$1.38$\bullet$ &93.65$\pm$1.25$\bullet$ &112.36$\pm$0.94$\bullet$ &130.84$\pm$1.65$\bullet$ &149.94$\pm$1.48$\bullet$ &168.94$\pm$1.87$\bullet$ &6\\

MR-NSGA-II-2r &75.64$\pm$1.40$\bullet$ &94.05$\pm$1.35$\bullet$ &112.77$\pm$0.97$\bullet$ &131.13$\pm$1.55$\bullet$ &150.08$\pm$1.52$\bullet$ &169.03$\pm$1.81$\bullet$ &5\\

MR-NSGA-II-100 &76.13$\pm$1.20$\bullet$ &94.84$\pm$1.27$\bullet$ &113.38$\pm$1.22$\bullet$ &131.96$\pm$1.50$\bullet$ &150.96$\pm$1.69$\bullet$ &169.82$\pm$1.96$\bullet$ &4\\

MR-MOEA/D &76.83$\pm$1.48$\bullet$ &95.47$\pm$1.36$\bullet$ &114.04$\pm$1.37$\bullet$ &132.27$\pm$1.53$\bullet$ &151.11$\pm$1.76$\bullet$ &170.16$\pm$1.89$\bullet$ &2.5\\

MR-GSEMO &76.77$\pm$1.40$\bullet$ &95.27$\pm$1.38$\bullet$ &113.95$\pm$1.12$\bullet$ &132.56$\pm$1.60$\bullet$ &151.25$\pm$1.74$\bullet$ &170.18$\pm$1.80$\bullet$ &2.5\\

MR-GSEMO-SR &\bf{77.95$\pm$1.49}  	&\bf{96.99$\pm$1.40} 	&\bf{115.76$\pm$1.33} 	&\bf{134.23$\pm$1.53}	&\bf{153.29$\pm$1.71} 	&\bf{172.23$\pm$1.87} &1\\
\hline\noalign{}
\multicolumn{8}{c}{Coordination Model}\\
\hline
Additive &68.91$\pm$2.41$\bullet$ &85.55$\pm$3.65$\bullet$ &103.68$\pm$5.30$\bullet$ &117.29$\pm$3.42$\bullet$	&136.53$\pm$4.38$\bullet$ &155.81$\pm$5.36$\bullet$ &7\\

Greedy &80.47$\pm$1.46$\bullet$	&100.40$\pm$1.39$\bullet$ &120.62$\pm$1.16$\bullet$	&140.53$\pm$1.81$\bullet$	&160.92$\pm$1.42$\bullet$	&181.15$\pm$2.05$\bullet$ &6\\

MR-NSGA-II-2r&81.29$\pm$1.52$\bullet$ &101.24$\pm$1.34$\bullet$ &121.51$\pm$1.01$\bullet$ &141.31$\pm$1.63$\bullet$ &161.88$\pm$1.68$\bullet$ &181.93$\pm$2.23$\bullet$ &5\\

MR-NSGA-II-100 &81.73$\pm$1.42$\bullet$ &101.82$\pm$1.32$\bullet$ &121.90$\pm$1.21$\bullet$ &141.69$\pm$1.77$\bullet$ &162.19$\pm$1.36$\bullet$ &182.25$\pm$2.14$\bullet$ &4\\

MR-MOEA/D &83.11$\pm$1.74$\bullet$	&103.78$\pm$1.61$\bullet$	&123.99$\pm$1.30$\bullet$	&143.68$\pm$1.44$\bullet$	&164.15$\pm$1.68$\bullet$	&184.09$\pm$2.06$\bullet$ &2.1\\

MR-GSEMO &82.87$\pm$1.59$\bullet$ &103.57$\pm$1.70$\bullet$ &123.64$\pm$1.11$\bullet$	&143.60$\pm$1.62$\bullet$	&164.15$\pm$1.55$\bullet$	&183.89$\pm$2.05$\bullet$ &2.9\\

MR-GSEMO-SR &\bf{83.76$\pm$1.73}  	&\bf{104.76$\pm$1.89}  	&\bf{124.79$\pm$1.31} &\bf{144.66$\pm$1.53} &\bf{164.98$\pm$1.61} &\bf{184.92$\pm$2.09} &1\\
\hline
\end{tabular}}
\end{center}
\end{table*}
\begin{table*}[ht!]\caption{Expected number of employed migrants (mean$\pm$std) obtained by the algorithms when the number of localities $|L|\in\{16, 18, \ldots,30\}$. For each $|L|$, the largest number is bolded, and `$\bullet$' denotes that MR-GSEMO-SR is significantly better than the corresponding algorithm by the Wilcoxon signed-rank test with confidence level $0.05$. Avg.R. denotes the average rank (the smaller, the better) of each algorithm under each setting as in~\cite{DBLP:journals/jmlr/Demsar06}.}\label{tab_L}
\scriptsize
\begin{center}
\setlength{\tabcolsep}{0.75mm}{
\begin{tabular}{c|llllllll|c}
\hline\noalign{}
\multicolumn{10}{c}{Interview Model}\\
\hline
$|L|$   & \multicolumn{1}{c}{$16$} &\multicolumn{1}{c}{$18$} & \multicolumn{1}{c}{$20$} & \multicolumn{1}{c}{$22$} & \multicolumn{1}{c}{$24$}&
\multicolumn{1}{c}{$26$} & \multicolumn{1}{c}{$28$}&
\multicolumn{1}{c}{$30$} & Avg.R.\\
\hline
Additive &63.61$\pm$2.58$\bullet$	&62.80$\pm$2.00$\bullet$	&60.81$\pm$2.62$\bullet$	&61.75$\pm$3.22$\bullet$	&60.39$\pm$2.63$\bullet$	&61.17$\pm$2.57$\bullet$	&58.91$\pm$1.77$\bullet$	&58.87$\pm$1.97$\bullet$ &7\\

Greedy  &70.97$\pm$2.31$\bullet$ &69.59$\pm$2.26$\bullet$ &68.28$\pm$2.03$\bullet$ &66.82$\pm$1.97$\bullet$ &65.03$\pm$1.37$\bullet$ &64.70$\pm$1.40$\bullet$ &63.45$\pm$1.49$\bullet$	&61.88$\pm$1.59$\bullet$ &6\\

MR-NSGA-II-2r &71.61$\pm$2.08$\bullet$  &70.13$\pm$2.29$\bullet$ &68.82$\pm$1.99$\bullet$ &67.56$\pm$1.88$\bullet$ &65.54$\pm$1.44$\bullet$ &65.18$\pm$1.65$\bullet$ &63.98$\pm$1.37$\bullet$ &62.65$\pm$1.54$\bullet$ &5\\

MR-NSGA-II-100 	&73.00$\pm$2.38$\bullet$	&71.25$\pm$2.35$\bullet$	&70.23$\pm$2.23$\bullet$	&68.84$\pm$2.08$\bullet$	&66.84$\pm$1.24$\bullet$	&66.90$\pm$1.62$\bullet$	&65.62$\pm$1.57$\bullet$	&64.01$\pm$1.39$\bullet$ &4\\

MR-MOEA/D 		&74.51$\pm$2.35$\bullet$	&72.84$\pm$2.38$\bullet$	&71.76$\pm$1.95$\bullet$	&70.26$\pm$2.15$\bullet$	&68.60$\pm$1.44$\bullet$	&68.23$\pm$1.62$\bullet$	&67.61$\pm$1.62$\bullet$	&65.31$\pm$1.30$\bullet$ &2\\

MR-GSEMO 	&73.87$\pm$2.36$\bullet$	&72.38$\pm$2.33$\bullet$	&71.05$\pm$2.07$\bullet$	&69.70$\pm$2.17$\bullet$	&67.91$\pm$1.60$\bullet$	&67.74$\pm$1.53$\bullet$	&66.47$\pm$1.47$\bullet$	&64.84$\pm$1.32$\bullet$ &3\\

MR-GSEMO-SR	&\bf{75.83$\pm$2.45}	&\bf{74.17$\pm$2.49}	&\bf{72.96$\pm$2.01}	&\bf{71.40$\pm$2.12}	&\bf{69.54$\pm$1.65}	&\bf{69.36$\pm$1.63}	&\bf{68.07$\pm$1.64}	&\bf{66.25$\pm$1.39} &1\\

\hline\noalign{}
\multicolumn{10}{c}{Coordination Model}\\
\hline
Additive 	&66.96$\pm$2.99$\bullet$	&65.57$\pm$2.48$\bullet$	&64.03$\pm$2.74$\bullet$	&63.97$\pm$2.93$\bullet$	&63.96$\pm$1.93$\bullet$	&63.12$\pm$1.83$\bullet$	&61.86$\pm$2.10$\bullet$	&61.32$\pm$2.24$\bullet$  &7\\

Greedy 	&75.40$\pm$2.37$\bullet$	&73.94$\pm$2.37$\bullet$	&72.35$\pm$2.16$\bullet$	&70.57$\pm$2.24$\bullet$	&68.75$\pm$1.36$\bullet$	&68.32$\pm$1.51$\bullet$	&66.96$\pm$1.51$\bullet$	&65.21$\pm$1.77$\bullet$ &6\\

MR-NSGA-II-2r	&75.96$\pm$2.28$\bullet$ &74.49$\pm$2.06$\bullet$ &73.34$\pm$2.04$\bullet$ &71.58$\pm$2.16$\bullet$ &69.55$\pm$1.36$\bullet$ &69.04$\pm$1.54$\bullet$ &67.65$\pm$1.52$\bullet$ &65.99$\pm$1.99$\bullet$ &5\\

MR-NSGA-II-100	&77.16$\pm$2.21$\bullet$	&75.57$\pm$2.28$\bullet$	&74.13$\pm$1.98$\bullet$	&72.52$\pm$2.08$\bullet$	&70.62$\pm$1.51$\bullet$	&70.07$\pm$1.35$\bullet$	&68.69$\pm$1.75$\bullet$	&67.07$\pm$1.45$\bullet$ &4\\

MR-MOEA/D &79.52$\pm$2.46$\bullet$	&77.54$\pm$2.61$\bullet$	&75.98$\pm$2.38$\bullet$	&74.19$\pm$2.24$\bullet$	&72.16$\pm$1.75$\bullet$	&71.69$\pm$1.73$\bullet$	&70.62$\pm$1.81$\bullet$	&68.68$\pm$1.58$\bullet$ &2\\

MR-GSEMO	&79.20$\pm$2.33$\bullet$	&77.23$\pm$2.56$\bullet$	&75.75$\pm$2.25$\bullet$	&73.90$\pm$2.38$\bullet$	&71.79$\pm$1.89$\bullet$	&71.09$\pm$1.52$\bullet$	&69.69$\pm$1.87$\bullet$	&67.86$\pm$1.49$\bullet$ &3\\

MR-GSEMO-SR	&\bf{81.19$\pm$2.80} &\bf{79.15$\pm$2.78}	&\bf{77.60$\pm$2.38}	&\bf{75.70$\pm$2.45}	&\bf{73.55$\pm$1.96}	&\bf{73.10$\pm$1.76}	&\bf{71.60$\pm$1.89}	&\bf{69.47$\pm$1.60} &1\\

\hline
\end{tabular}}
\end{center}
\end{table*}

\begin{table*}[t!]\caption{Expected number of employed migrants (mean$\pm$std) obtained by the algorithms when the number of jobs $|J|\in\{60,70,80,90\}\cup\{110,120,130,140\}$. For each $|J|$, the largest number is bolded, and `$\bullet$' denotes that MR-GSEMO-SR is significantly better than the corresponding algorithm by the Wilcoxon signed-rank test with confidence level $0.05$. Avg.R. denotes the average rank (the smaller, the better) of each algorithm under each setting as in~\cite{DBLP:journals/jmlr/Demsar06}.}\label{tab_J}
\scriptsize
\begin{center}
\setlength{\tabcolsep}{0.75mm}{
\begin{tabular}{c|llllllll|c}
\hline\noalign{}
\multicolumn{10}{c}{Interview Model}\\
\hline
$|J|$  & \multicolumn{1}{c}{$60$ } & \multicolumn{1}{c}{$70$} & \multicolumn{1}{c}{$80$} &\multicolumn{1}{c}{$90$}  &\multicolumn{1}{c}{$110$ }&\multicolumn{1}{c}{$120$} &\multicolumn{1}{c}{$130$} &\multicolumn{1}{c}{$140$} &Avg.R.\\
\hline
Additive &36.75$\pm$4.24$\bullet$	&41.81$\pm$2.47$\bullet$	&50.25$\pm$2.61$\bullet$	&56.46$\pm$3.64$\bullet$		&66.15$\pm$2.47$\bullet$	&70.95$\pm$3.05$\bullet$	&73.46$\pm$2.44$\bullet$	&75.28$\pm$3.23$\bullet$ &7\\

Greedy &47.98$\pm$1.22$\bullet$	&56.69$\pm$1.74$\bullet$	&64.88$\pm$1.82$\bullet$
&70.25$\pm$1.75$\bullet$	&76.96$\pm$1.60$\bullet$	&78.36$\pm$2.46$\bullet$	&79.81$\pm$2.16$\bullet$	&80.42$\pm$1.61$\bullet$ &6\\

MR-NSGA-II-2r &48.77$\pm$1.07$\bullet$ &57.30$\pm$1.65$\bullet$ &65.21$\pm$1.61$\bullet$ &70.55$\pm$1.81$\bullet$ &77.19$\pm$1.75$\bullet$ &78.60$\pm$2.46$\bullet$ &80.74$\pm$1.86$\bullet$ &80.55$\pm$1.80$\bullet$ &5\\

MR-NSGA-II-100 &49.18$\pm$1.21$\bullet$	&57.95$\pm$1.40$\bullet$	&66.12$\pm$1.67$\bullet$	&71.39$\pm$1.58$\bullet$		&77.94$\pm$1.59$\bullet$	&79.99$\pm$2.21$\bullet$	&81.05$\pm$1.72$\bullet$	&81.70$\pm$1.45$\bullet$ &4\\

MR-MOEA/D &49.42$\pm$1.17$\bullet$	&58.34$\pm$1.35$\bullet$	&66.78$\pm$1.46$\bullet$	&72.51$\pm$1.63$\bullet$		&79.18$\pm$1.51$\bullet$	&81.35$\pm$2.28$\bullet$	&82.59$\pm$1.93$\bullet$	&83.43$\pm$1.47$\bullet$ &2\\

MR-GSEMO &49.32$\pm$1.16$\bullet$	&58.25$\pm$1.43$\bullet$	&66.70$\pm$1.55$\bullet$	&72.11$\pm$1.74$\bullet$		&78.70$\pm$1.58$\bullet$	&80.63$\pm$2.43$\bullet$	&82.00$\pm$1.70$\bullet$	&82.43$\pm$1.58$\bullet$ &3\\

MR-GSEMO-SR &\bf{50.03$\pm$1.19} 	&\bf{59.40$\pm$1.43} 	&\bf{68.23$\pm$1.42} 	&\bf{73.84$\pm$1.65} 		&\bf{81.00$\pm$1.58} 	&\bf{83.14$\pm$2.09} 	&\bf{84.33$\pm$1.80} 	&\bf{85.46$\pm$1.73} &1\\
\hline\noalign{}
\multicolumn{10}{c}{Coordination Model}\\
\hline
Additive &37.99$\pm$4.17$\bullet$	&43.41$\pm$2.43$\bullet$	&52.09$\pm$3.91$\bullet$	&58.80$\pm$4.01$\bullet$		&69.50$\pm$2.80$\bullet$	&74.13$\pm$4.12$\bullet$	&77.10$\pm$2.69$\bullet$	&79.10$\pm$3.64$\bullet$ &7\\

Greedy &50.38$\pm$1.76$\bullet$	&59.29$\pm$2.13$\bullet$	&67.99$\pm$1.99$\bullet$	&74.25$\pm$1.65$\bullet$		&79.85$\pm$1.61$\bullet$	&81.38$\pm$2.70$\bullet$	&81.94$\pm$2.59$\bullet$	&81.84$\pm$2.30$\bullet$ &6\\

MR-NSGA-II-2r &51.51$\pm$1.90$\bullet$ &60.13$\pm$2.26$\bullet$ &68.75$\pm$2.11$\bullet$ &74.85$\pm$1.75$\bullet$ &81.13$\pm$1.48$\bullet$ &82.27$\pm$2.70$\bullet$ &82.82$\pm$2.03$\bullet$ &83.05$\pm$1.76$\bullet$ &5\\

MR-NSGA-II-100 &51.95$\pm$1.63$\bullet$	&60.86$\pm$1.76$\bullet$	&69.45$\pm$1.96$\bullet$	&75.43$\pm$1.93$\bullet$		&81.65$\pm$1.66$\bullet$	&83.37$\pm$2.35$\bullet$	&84.31$\pm$1.87$\bullet$	&84.31$\pm$1.79$\bullet$ &4\\

MR-MOEA/D &52.71$\pm$1.36$\bullet$	&62.08$\pm$1.73$\bullet$	&71.20$\pm$1.75$\bullet$	&77.87$\pm$2.04$\bullet$		&84.56$\pm$1.76$\bullet$	&86.38$\pm$2.30$\bullet$	&87.21$\pm$2.25$\bullet$	&88.30$\pm$1.71$\bullet$ &2\\

MR-GSEMO &52.27$\pm$1.49$\bullet$	&61.58$\pm$1.59$\bullet$	&70.49$\pm$1.88$\bullet$	&77.31$\pm$1.73$\bullet$		&83.68$\pm$1.63$\bullet$	&85.05$\pm$2.50$\bullet$	&86.00$\pm$2.09$\bullet$	&86.88$\pm$1.75$\bullet$ &3\\

MR-GSEMO-SR &\bf{53.29$\pm$1.48}  	&\bf{62.83$\pm$1.85} 	&\bf{72.33$\pm$1.78} 	&\bf{79.33$\pm$1.89} 		&\bf{86.23$\pm$1.71} 	&\bf{88.11$\pm$2.19} 	&\bf{88.70$\pm$1.97} 	&\bf{89.60$\pm$1.71} &1\\
\hline
\end{tabular}}
\end{center}
\end{table*}

\begin{table*}[t!]\caption{Expected number of employed migrants (mean$\pm$std) obtained by the algorithms when the number of professions $|\Pi|\in\{5,10,\ldots,30\}$. For each $|\Pi|$, the largest number is bolded, and `$\bullet$' denotes that MR-GSEMO-SR is significantly better than the corresponding algorithm by the Wilcoxon signed-rank test with confidence level $0.05$. Avg.R. denotes the average rank (the smaller, the better) of each algorithm under each setting as in~\cite{DBLP:journals/jmlr/Demsar06}.}\label{tab_Pi}
\scriptsize
\begin{center}
\setlength{\tabcolsep}{0.75mm}{
\begin{tabular}{c|llllll|c}
\hline\noalign{}
\multicolumn{8}{c}{Interview Model}\\
\hline
$|\Pi|$  & \multicolumn{1}{c}{$5$} & \multicolumn{1}{c}{$10$} &\multicolumn{1}{c}{$15$ }&\multicolumn{1}{c}{$20$} &\multicolumn{1}{c}{$25$} &\multicolumn{1}{c}{$30$} &Avg.R.\\
\hline
Additive 	&49.99$\pm$3.94$\bullet$		&42.81$\pm$2.74$\bullet$	&42.82$\pm$2.65$\bullet$	&42.32$\pm$2.10$\bullet$	&43.17$\pm$2.68$\bullet$	&43.75$\pm$1.84$\bullet$ &7\\

Greedy 	&65.97$\pm$2.24$\bullet$		&58.59$\pm$1.74$\bullet$	&55.71$\pm$1.30$\bullet$	&55.39$\pm$2.17$\bullet$	&54.48$\pm$2.10$\bullet$	&52.58$\pm$1.46$\bullet$ &6\\

MR-NSGA-II-2r &66.70$\pm$2.01$\bullet$ &59.26$\pm$1.93$\bullet$ &56.18$\pm$1.39$\bullet$ &55.76$\pm$2.23$\bullet$ &54.82$\pm$2.11$\bullet$ &52.78$\pm$1.46$\bullet$ &5\\

MR-NSGA-II-100 &67.55$\pm$2.02$\bullet$	&60.37$\pm$1.84$\bullet$	&56.95$\pm$1.47$\bullet$	&56.42$\pm$2.16$\bullet$	&55.20$\pm$1.98$\bullet$	&53.14$\pm$1.50$\bullet$ &4\\

MR-MOEA/D 	&68.69$\pm$2.14$\bullet$		&61.14$\pm$1.96$\bullet$	&57.59$\pm$1.66$\bullet$	&56.82$\pm$2.06$\bullet$	&55.67$\pm$2.11$\bullet$	&53.16$\pm$1.47$\bullet$ &2.1\\

MR-GSEMO 	&68.50$\pm$2.33$\bullet$		&60.98$\pm$1.93$\bullet$	&57.55$\pm$1.54$\bullet$	&56.64$\pm$2.05$\bullet$	&55.59$\pm$2.09$\bullet$	&53.31$\pm$1.48$\bullet$ &2.9\\

MR-GSEMO-SR 	&\bf{69.92$\pm$2.09} 	&\bf{62.02$\pm$1.87} 	&\bf{58.29$\pm$1.64} 	&\bf{57.34$\pm$2.01} 	&\bf{56.01$\pm$2.14} 	&\bf{53.58$\pm$1.48} &1\\
\hline\noalign{}
\multicolumn{8}{c}{Coordination Model}\\
\hline
Additive &50.77$\pm$4.27$\bullet$		&44.47$\pm$2.28$\bullet$	&43.85$\pm$2.73$\bullet$	&42.93$\pm$2.33$\bullet$	&42.89$\pm$2.89$\bullet$	&44.15$\pm$1.73$\bullet$ &7\\

Greedy &69.80$\pm$2.50$\bullet$	&61.28$\pm$1.68$\bullet$	&57.65$\pm$1.41$\bullet$	&56.84$\pm$2.22$\bullet$	&55.84$\pm$2.08$\bullet$	&53.37$\pm$1.44$\bullet$ &6\\

MR-NSGA-II-2r &70.67$\pm$2.45$\bullet$ &61.93$\pm$1.77$\bullet$  &58.08$\pm$1.52$\bullet$  &57.08$\pm$2.03$\bullet$  &55.87$\pm$2.09$\bullet$  &53.46$\pm$1.48$\bullet$ &5\\

MR-NSGA-II-100 &71.60$\pm$2.12$\bullet$	&62.53$\pm$1.74$\bullet$	&58.44$\pm$1.48$\bullet$	&57.43$\pm$2.16$\bullet$	&56.18$\pm$2.16$\bullet$	&53.54$\pm$1.58$\bullet$ &4\\

MR-MOEA/D 	&72.38$\pm$2.39$\bullet$		&63.16$\pm$2.00$\bullet$	&58.89$\pm$1.73$\bullet$	&57.73$\pm$2.15$\bullet$	&56.46$\pm$2.18$\bullet$	&53.77$\pm$1.58$\bullet$ &2.3\\

MR-GSEMO 	&72.27$\pm$2.53$\bullet$		&63.15$\pm$2.04$\bullet$	&59.01$\pm$1.56$\bullet$	&57.73$\pm$2.19$\bullet$	&56.45$\pm$2.24$\bullet$	&53.83$\pm$1.54$\bullet$ &2.7\\

MR-GSEMO-SR &\bf{73.75$\pm$2.44} 	&\bf{64.13$\pm$2.02} 	&\bf{59.61$\pm$1.80} 	&\bf{58.24$\pm$2.05} 	&\bf{56.76$\pm$2.22} 	&\bf{54.04$\pm$1.64} &1\\
\hline
\end{tabular}}
\end{center}
\end{table*}

\begin{figure*}[ht]\centering
\begin{minipage}[c]{0.24\linewidth}\centering
        \includegraphics[width=1\linewidth]{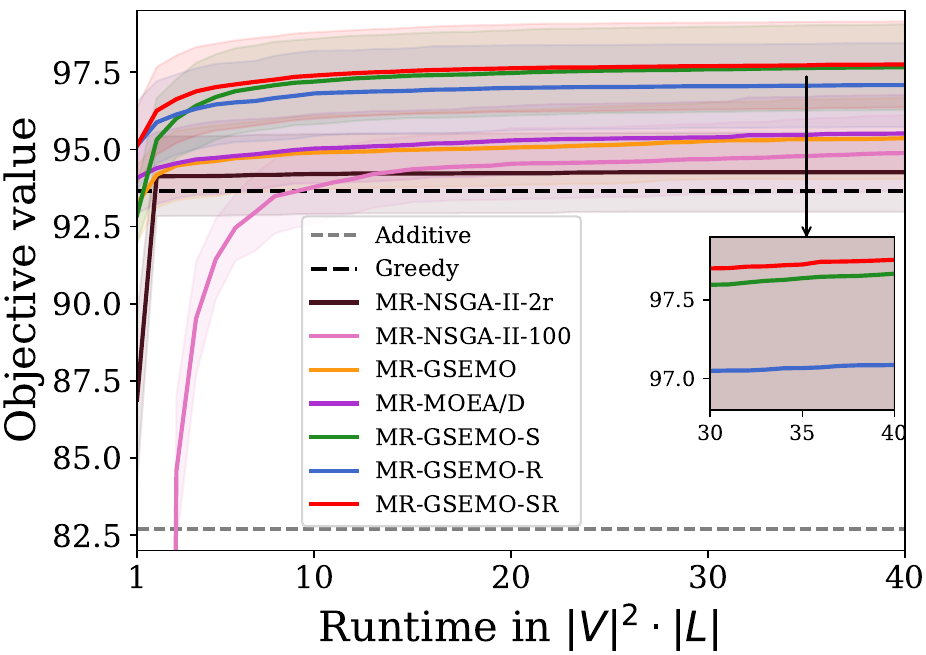}
\end{minipage}
\begin{minipage}[c]{0.24\linewidth}\centering
        \includegraphics[width=1\linewidth]{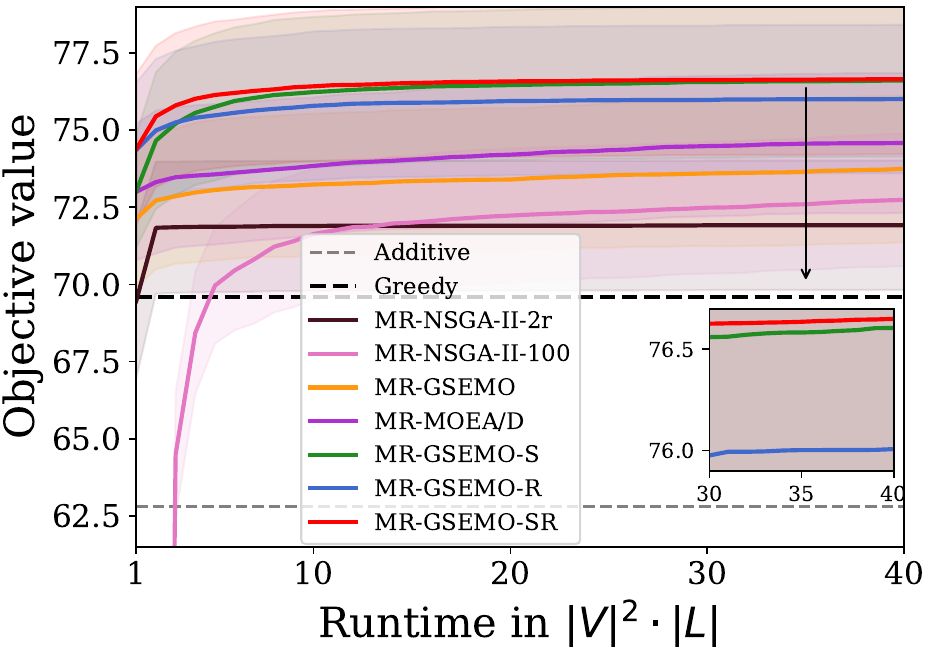}
\end{minipage}
\begin{minipage}[c]{0.24\linewidth}\centering
        \includegraphics[width=1\linewidth]{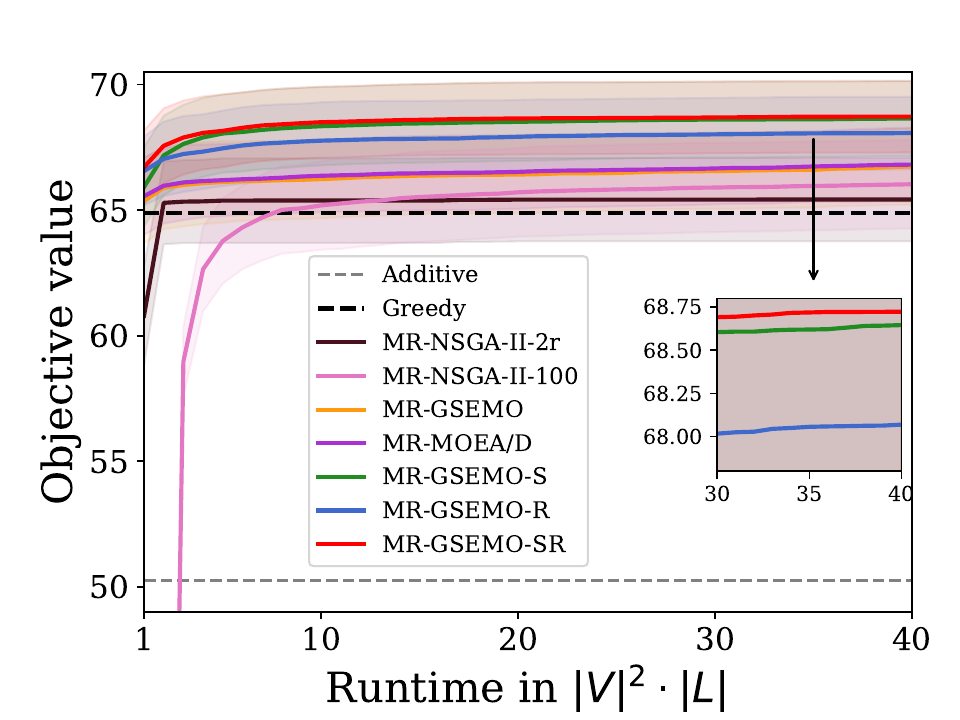}
\end{minipage}
\begin{minipage}[c]{0.24\linewidth}\centering
        \includegraphics[width=1\linewidth]{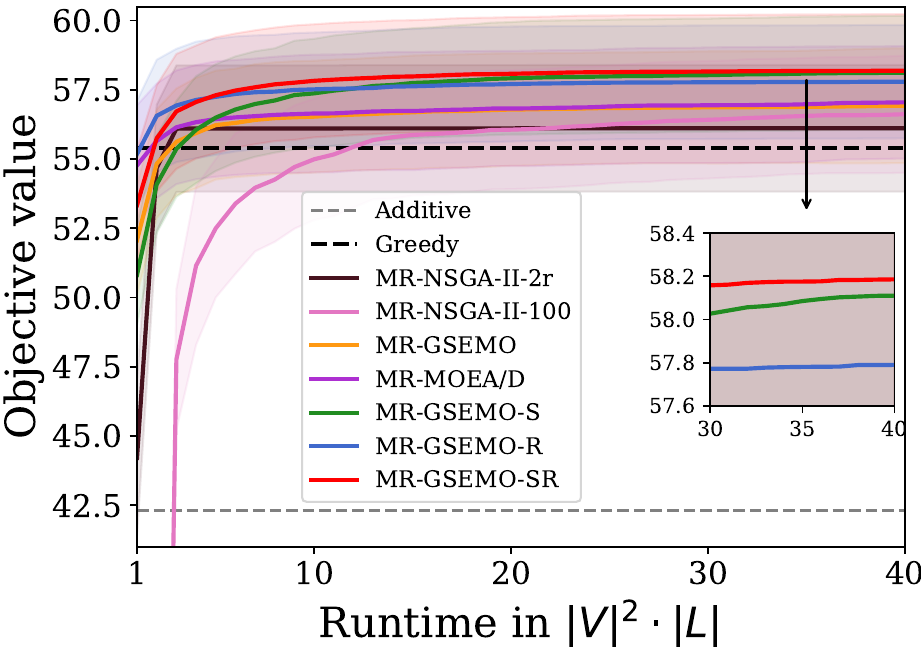}
\end{minipage}\\\vspace{0.6em}
\begin{minipage}[c]{0.24\linewidth}\centering
    \small (a) $|V|=120$ \\Interview Model
\end{minipage}
\begin{minipage}[c]{0.24\linewidth}\centering
    \small(b) $|L|=16$ \\Interview Model
\end{minipage}
\begin{minipage}[c]{0.24\linewidth}\centering
    \small(c) $|J|=80$ \\Interview Model
\end{minipage}
\begin{minipage}[c]{0.24\linewidth}\centering
    \small(d) $|\Pi|=20$ \\Interview Model
\end{minipage}\\\vspace{0.6em}
\begin{minipage}[c]{0.24\linewidth}\centering
        \includegraphics[width=1\linewidth]{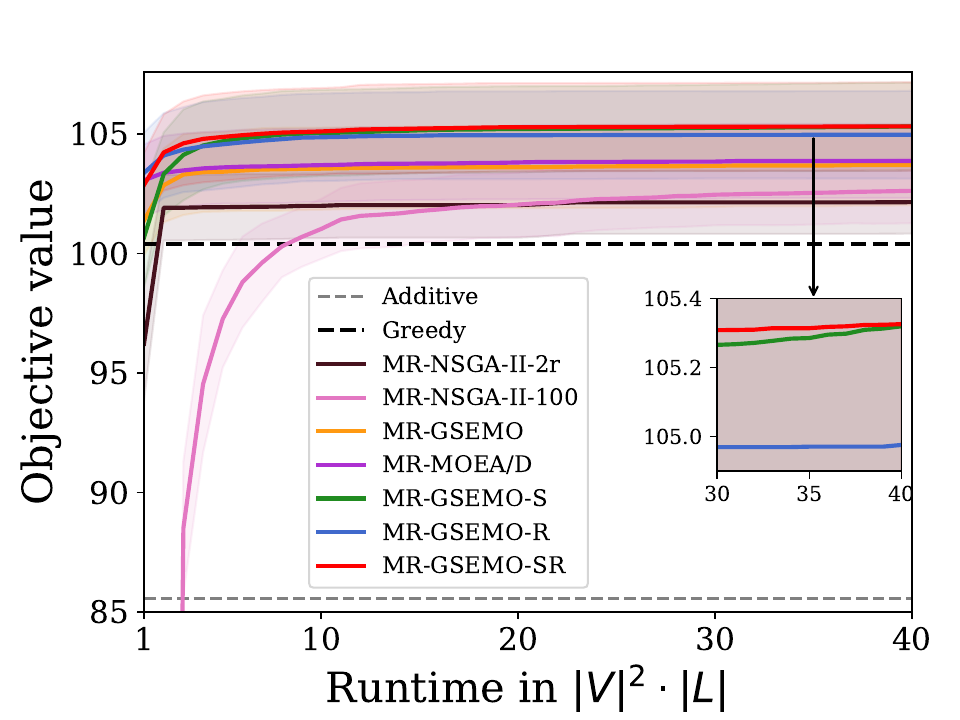}
\end{minipage}
\begin{minipage}[c]{0.24\linewidth}\centering
        \includegraphics[width=1\linewidth]{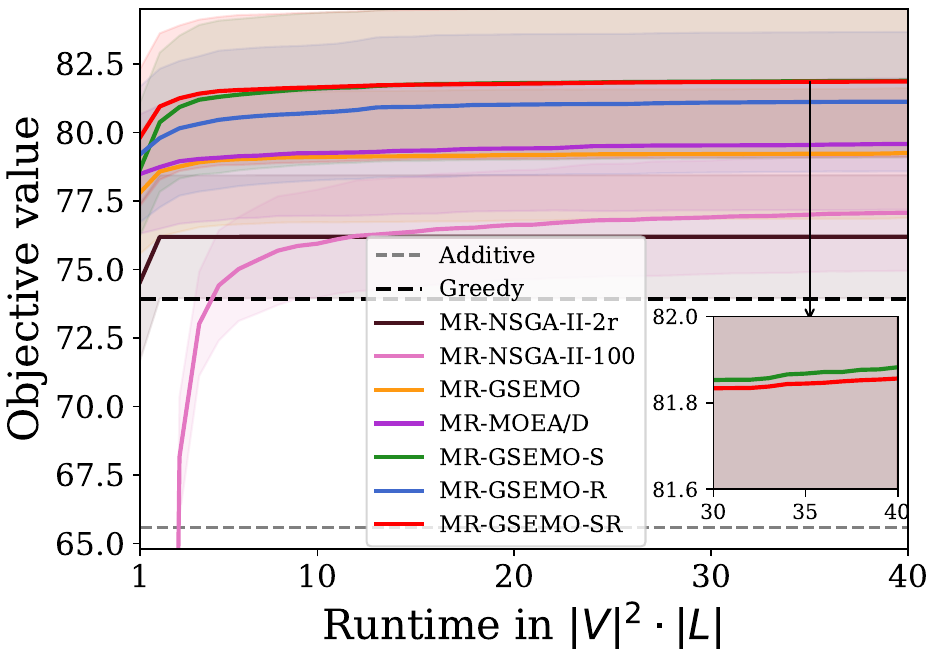}
\end{minipage}
\begin{minipage}[c]{0.24\linewidth}\centering
        \includegraphics[width=1\linewidth]{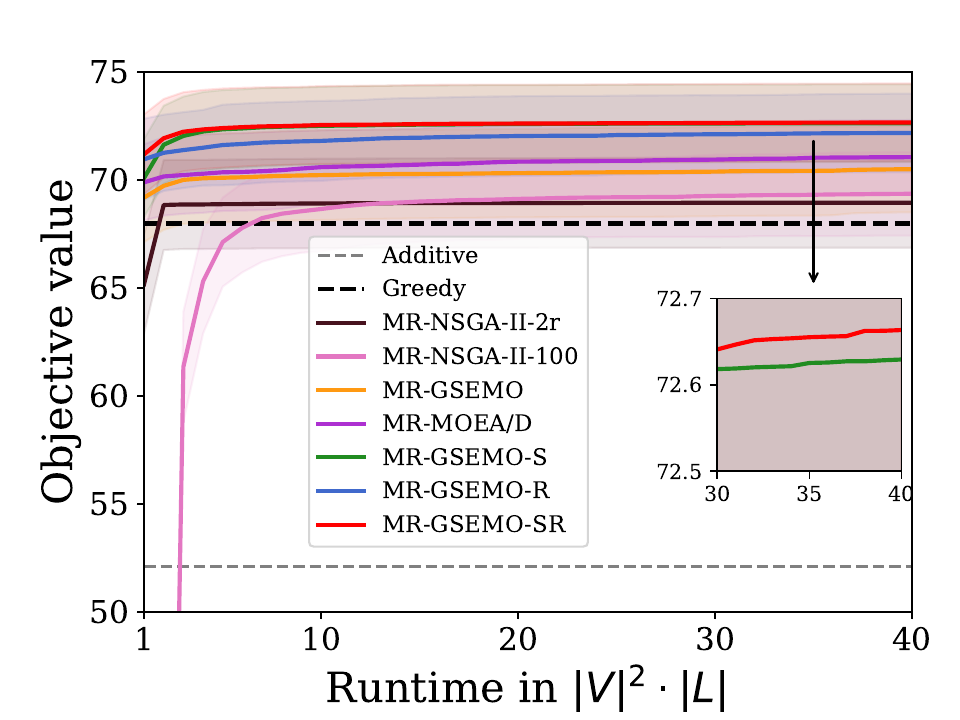}
\end{minipage}
\begin{minipage}[c]{0.24\linewidth}\centering
        \includegraphics[width=1\linewidth]{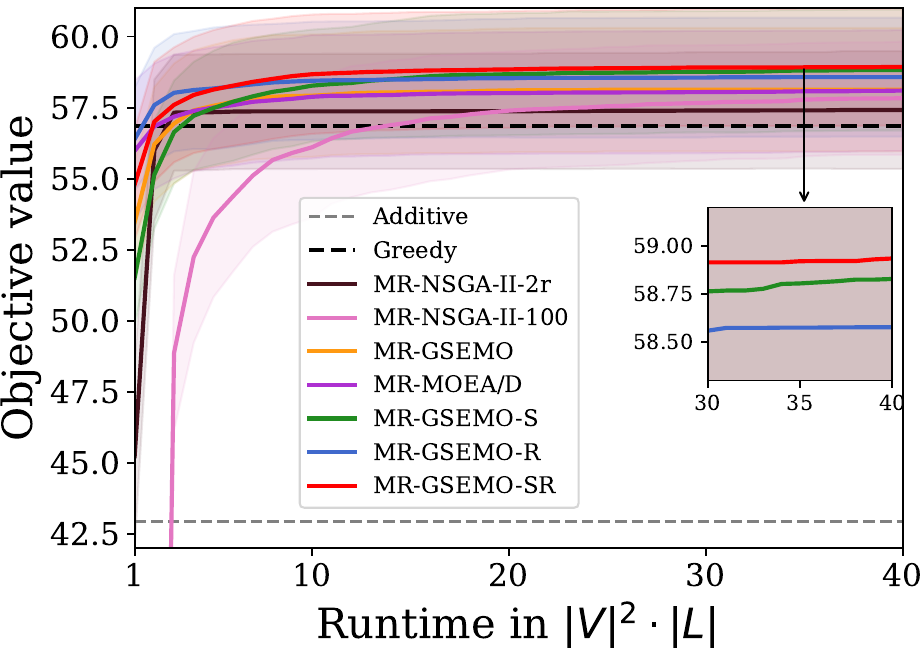}
\end{minipage}\\\vspace{0.6em}
\begin{minipage}[c]{0.24\linewidth}\centering
    \small(e) $|V|=120$ \\Coordination Model
\end{minipage}
\begin{minipage}[c]{0.24\linewidth}\centering
    \small(f) $|L|=16$ \\Coordination Model
\end{minipage}
\begin{minipage}[c]{0.24\linewidth}\centering
    \small(g) $|J|=80$ \\Coordination Model
\end{minipage}
\begin{minipage}[c]{0.24\linewidth}\centering
    \small(h) $|\Pi|=20$ \\Coordination Model
\end{minipage}
\caption{One each migration model, the objective value vs. runtime (i.e., number of objective evaluations) with $|V|=120$, $|L|=16$, $|J|=80$ and $|\Pi|=20$, respectively.}\label{fig_time}
\end{figure*}

Among the five variants of MR-EMO, MR-NSGA-II-2r performs the worst, which may be because the population of NSGA-II-2r may contain redundant dominated solutions, leading to the bad performance; MR-NSGA-II-100 slightly outperforms MR-NSGA-II-2r, which may be because MR-NSGA-II-100, with a smaller population size, can more easily select superior parents. This ease of selection facilitates the generation of higher-quality offspring, which, in turn, enhances the overall results within the same number of generations. However, MR-NSGA-II-2r, MR-NSGA-II-100 and MR-GSEMO all perform worse than MR-MOEA/D. This improvement may be attributed to MR-MOEA/D's neighborhood-based search strategy, which enhances the local search effectiveness within each subproblem's area, thereby boosting both search efficiency and solution quality. MR-GSEMO-SR always achieves the best average objective value (i.e., the largest expected number of employed migrants), and is significantly better than all the other algorithms by the Wilcoxon signed-rank test~\cite{wilcoxon1945individual} with confidence level $0.05$. These observations show the effectiveness of the introduced matrix-swap mutation and repair mechanism for the migrant resettlement problem. 

Tables~\ref{tab_V} to~\ref{tab_Pi} also show that the expected number of employed migrants achieved by each algorithm increases with the number $|V|$ of migrants and the number $|J|$ of jobs, but decreases with the number $|L|$ of localities and the number $|\Pi|$ of professions. More migrants and jobs will bring more employment naturally, while more localities and professions with a fixed number of migrants and the number of jobs will increase the difficulty of resettlement and thus lead to less employment.

\textbf{Ablation Study.} Compared with MR-GSEMO using bit-wise mutation only, MR-GSEMO-SR employs matrix-swap mutation and repair mechanism additionally. To examine the utility of these two introduced components more clearly, we run another two variants MR-GSEMO-S and MR-GSEMO-R, which apply only matrix-swap mutation and repair mechanism, respectively. We plot the curve of objective value over runtime under two migration models with $|V|=120$, $|L|=16$, $|J|=80$ and $|\Pi|=20$, respectively, as shown in Figure~\ref{fig_time}. The additive and greedy algorithms are fixed-time algorithms, while the variants of MR-EMO are anytime algorithms, and can get better performance by using more time. All variants of MR-EMO (i.e., MR-NSGA-II-2r, MR-NSGA-II-100, MR-MOEA/D, MR-GSEMO, MR-GSEMO-SR, MR-GSEMO-S and MR-GSEMO-R) surpass the additive and greedy algorithms within $5\cdot|V|^2\cdot|L|$ time. We can observe that even applying only one component (i.e., matrix-swap mutation or repair mechanism), both MR-GSEMO-S and MR-GSEMO-R are still better than MR-NSGA-II-2r, MR-NSGA-II-100, MR-MOEA/D and MR-GSEMO, showing the effectiveness of matrix-swap mutation and repair mechanism clearly. Compared with one-point crossover and bit-wise mutation, the matrix-swap mutation operator has a larger probability of generating feasible solutions, and thus can make feasible regions explored efficiently. Instead of discarding infeasible solutions directly, repairing them may lead to good feasible solutions. Meanwhile, MR-GSEMO-S is better than MR-GSEMO-R, disclosing the more important role of the matrix-swap mutation operator. Finally, MR-GSEMO-SR, which employs both matrix-swap mutation and repair mechanism, performs the best, implying that using the two components together can further improve the performance.

\section{Conclusion}\label{conclusion}

In this paper, we propose the general framework MR-EMO for the important migrant resettlement problem, which tries to maximize the expected number of employed migrants while dispatching as few migrants as possible. MR-EMO can be equipped with any MOEA to solve the reformulated bi-objective problem, and we apply NSGA-II, MOEA/D and GSEMO, and also propose the specific MOEA, GSEMO-SR, by introducing matrix-swap mutation and repair mechanism into GSEMO. We prove that using either GSEMO or GSEMO-SR, MR-EMO can achieve the approximation guarantee of $1/(k+\frac{1}{p}+\frac{2\epsilon r}{1-\epsilon})$, surpassing the best-known one~\cite{golz2019migration}, obtained by the greedy algorithm. It is more challenging to analyze the theoretical guarantees of MR-NSGA-II and MR-MOEA/D. Future work could focus on establishing these guarantees, which might reveal unique strengths or weaknesses that could guide their application and further development. This would be an interesting area for exploration.

By experiments under the two migration models, i.e., interview and coordination models, we show that MR-EMO using any of the four MOEAs can produce better employment than the previous (i.e., additive and greedy) algorithms, and using GSEMO-SR leads to the best performance, which is significantly better than all the other algorithms. The matrix-swap mutation operator and repair mechanism can also be incorporated into other MOEAs besides GSEMO. Thus, it is interesting to examine whether they can still bring performance improvement. In the future, exploring other advanced MOEAs or designing improved versions within the MR-EMO framework could lead to the development of even more robust algorithms for solving the migration resettlement problem.

The migrant resettlement problem studied in this paper aims to settle migrants in suitable localities, where there are a number of jobs available to migrants. The resettlement plan considers the matching between migrants and available local jobs, and the goal is to maximize the expected number of employed migrants. What we want to emphasize is that when applying the proposed framework MR-EMO to real-world data, it is worth generating more accurate matching probabilities during data preparation, by considering migrants’ own technology, migrant wishes, local policies that support migrants, etc. Practical application of the MR-EMO framework necessitates collaboration with migration agencies and governments to access accurate and comprehensive data. It is important to clarify that the resettlement process does not involve ethical or moral issues, such as racial or religious discrimination.

\bibliographystyle{abbrvnat}
\bibliography{migration}
\end{document}